\documentclass{article}
\usepackage{amsmath, amssymb, amscd, amsthm}
\usepackage{graphicx}
\usepackage{hyperref}
\usepackage{makecell}
\usepackage[utf8]{inputenc} % allow utf-8 input
\usepackage[T1]{fontenc}    % use 8-bit T1 fonts
\usepackage{hyperref}       % hyperlinks
\usepackage{multirow}
\usepackage{url}            % simple URL typesetting
\usepackage{booktabs}       % professional-quality tables
\usepackage{amsfonts}       % blackboard math symbols
\usepackage{nicefrac}       % compact symbols for 1/2, etc.
\usepackage{microtype}      % microtypography
\usepackage[dvipsnames]{xcolor}
\usepackage{subcaption}
\usepackage{bm}
\usepackage[ruled, lined, linesnumbered, commentsnumbered, longend]{algorithm2e}

\usepackage[capitalize,noabbrev]{cleveref}
\theoremstyle{plain}

\newcommand{\cX}{\mathcal{X}}
\newcommand{\cY}{\mathcal{Y}}
\newcommand{\expe}{\exp_\varepsilon}
\newcommand{\norm}[1]{\left\lVert#1\right\rVert}

\newtheorem{theorem}{Theorem}
\newtheorem{lemma}{Lemma}
\newtheorem{proposition}{Proposition}
\newtheorem{corollary}{Corollary}
\newtheorem{conjecture}{Conjecture}
\theoremstyle{definition}
\newtheorem{definition}{Definition}

\title{Estimation of entropy-regularized optimal transport maps between non-compactly supported measures}
\author{Matthew Werenski$^{(\circ) }$, James M. Murphy$^{(\ddagger)}$, Shuchin Aeron$^{(\dagger)}$
  \thanks{
\text{($\circ$): Department of Computer Science, Tufts University}, \text{$(\ddagger)$: Department of Mathematics, Tufts University}, \text{$(\dagger)$: Department of Electrical and Computer Engineering, Tufts University}} 
  } 
\begin{document}
\maketitle
\begin{abstract}
    This paper addresses the problem of estimating entropy-regularized optimal transport (EOT) maps with squared-Euclidean cost between source and target measures that are subGaussian. In the case that the target measure is compactly supported or strongly log-concave, we show that for a recently proposed in-sample estimator, the expected squared $L^2$-error decays at least as fast as $O(n^{-1/3})$ where $n$ is the sample size. For the general subGaussian case we show that the expected $L^1$-error decays at least as fast as $O(n^{-1/6})$, and in both cases we have polynomial dependence on the regularization parameter. While these results are suboptimal compared to known results in the case of compactness of both the source and target measures (squared $L^2$-error converging at a rate $O(n^{-1})$) and for when the source is subGaussian while the target is compactly supported (squared $L^2$-error converging at a rate $O(n^{-1/2})$), their importance lie in eliminating the compact support requirements. The proof technique makes use of a bias-variance decomposition where the variance is controlled using standard concentration of measure results and the bias is handled by T1-transport inequalities along with sample complexity results in estimation of EOT cost under subGaussian assumptions. Our experimental results point to a looseness in controlling the variance terms and we conclude by posing several open problems. 
\end{abstract}

\section{Introduction}

The theory and applications of optimal transport (OT) has found use in a variety of areas such as partial differential equations \cite{evans1997partial, villani2009optimal}, machine learning \cite{torres2021survey, singh2020model}, graphics \cite{bonneel2023survey}, and statistics \cite{hallin2022finite}, just to name a few. 
A version of the OT paradigm, which incorporates an entropy regularization term that is historically related to the Schr\"odinger Bridge  problem \cite{chen2014relation} was popularized recently in \cite{cuturi2013sinkhorn}. This is known as entropy-regularized optimal transport (EOT), and was originally proposed as a way of off-setting the computational burden of OT which typically requires $O(n^3)$ operations for measures supported on $n$ points \cite{pele2009fast} compared to (up to log factors) $O(n^2)$ operations for EOT \cite{altschuler2017near,lin2022efficiency}. In recent years, EOT has taken on significant importance in its own right owing to its connections with stochastic control \cite{chen2020stochastic}, theory and applications of generative modeling using stochastic diffusion \cite{de2021diffusion}, differentiable-ranks \cite{cuturi2019differentiable, blondel2020fast, werenski2023rank}, and its utility in approximating the OT maps \cite{pooladian2021entropic}. 

The focus of this work is on statistical aspects of EOT, studied in several recent papers indicating favorable (dimension-free) non-asymptotic convergence rates in estimation of associated quantities namely, the EOT-cost and the EOT-map \cite{genevay2019sample, mena2019statistical,rigollet2022sample,stromme23_ICML, MWMA_JMLR23, stromme2023minimum, groppe2023lower} where convergence rates are typically on the order of $n^{-1/k}$ for a small $k \geq 1$, independent of the dimension. These results stand in contrast to the dimension-dependent rates for unregularized OT,  where rates are typically of the order $n^{-1/d}$ or $n^{-2/d}$ unless structural assumptions are made on the measures involved \cite{weed2019sharp,Hutter_Rigollet2021}. In this paper, we specifically consider the case of squared-Euclidean cost and on estimation of the EOT-map from samples. In this context, current results require the target measure to be compactly supported \cite{MWMA_JMLR23,rigollet2022sample,stromme2023minimum, stromme23_ICML,pooladian2023minimax} which precludes them from being applied to a large class of measures. Towards extending the state-of-the-art and in particular eliminating the compactness assumption in estimation of EOT-maps we make the following main contributions. 

\subsection{Main Contributions}
\begin{enumerate}
    \item We extend the technical framework of splitting the analysis into analyzing a bias and variance term presented in \cite{stromme23_ICML}. Specifically we avoid a Schr\"odinger bridge construction and instead utilize information-transport inequalities coupled with finer analysis of the variance to control the estimation error for non-compactly supported measures; see Lemma \ref{lem:easy_term}.

    \item In Theorems \ref{thm:ultimate_bounded} and \ref{thm:ultimate_lcc} we utilize this framework to cover the cases where the target measure has either compact support or is strongly log-concave. We show that for these cases the expected squared $L^2$ error decays as fast as $O(n^{-1/3})$ with polynomial dependence on the regularization parameter.
    
    \item In Theorem \ref{thm:ultimate_subg} we apply a slightly modified argument to cover the case where the target measure is only subGaussian, albeit with a weaker guarantee compared to the strongly log-concave case, with the expected $L^1$ error decaying as fast as $O(n^{-1/6})$ with polynomial dependence on the regularization parameter.

    \item Our experimental results\footnote{Code to reproduce the experimental results is available at \url{https://github.com/MattWerenski/entropic-map}} point to a possible weakness in analysis of the variance term, where a potentially favorable dependence on the sample size is ignored. To this end we put forward Conjecture \ref{conj:variance} which if proven true would improve the error rates to $O(n^{-1/2})$ for compactly supported and strongly log-concave measures and $O(n^{-1/4})$ for the subGaussian measures.
\end{enumerate}

\section{Problem Set-Up, Notations, and Preliminaries} \label{sec:problem_set_up}

Let $\mathcal{P}_2(\mathbb{R}^d)$ be the set of probability measures with finite second moment and consider $\mu, \nu \in \mathcal{P}_2(\mathbb{R}^d)$. For notational convenience, for any functional $f:\mathbb{R}^d \rightarrow \mathbb{R}$ and a given measure $\mu$, let $\mu(f) = \mathbb{E}_{\bm{X} \sim \mu}[f(\bm{X})]$ with the same notation applied to joint measures. By $\mu \otimes \nu $ we mean the product measure on $\mathbb{R}^d \times \mathbb{R}^d$. 

We will use the notation $\expe(c) = \exp \left ( \frac{c}{\varepsilon} \right )$. We will write $a \lesssim b$ to mean that there is a constant $C > 0$ which may only depend on the dimension $d$ such that $a \leq Cb$. For a function $F:\mathbb{R}^d \rightarrow \mathbb{R}^d$ we will write $\norm{F}_{L^p(\mu)} = \mathbb{E}_{\bm{X} \sim \mu} \left [ \norm{F(\bm{X})}_2^p \right ]^{1/p}$.

Throughout this work we will consider random variables of the following type.
\begin{definition}
A random vector $\bm{X}$ on $\mathbb{R}^d$ is said to be \textit{$\sigma^2$-norm-subGaussian} if $
    \mathbb{E}\left [ \exp\left( \|\bm{X}\|^2 /2d\sigma^2\right)\right] \leq 2.$
and we define its norm by
\begin{equation*}
    \norm{\bm{X}}_{\psi_2} = \inf \left \{ \sigma > 0 \ \bigg | \ \mathbb{E}\left [ \exp\left( \frac{\|\bm{X}\|^2}{2d\sigma^2}\right)\right] \leq 2 \right \}.
\end{equation*}
\end{definition}
Norm-subGaussian random vectors were introduced in \cite{jin2019short} with an equivalent definition in terms of tail-bounds.

\subsection{Preliminaries on Entropy-Regularized Optimal Transport}

For $\varepsilon>0$ The EOT problem with squared-Euclidean cost between two probability measures $\mu,\nu \in \mathcal{P}_{2}(\mathbb{R}^d)$ is the following minimization:
\begin{align}
\label{eq:EOT}
    S_\varepsilon(\mu,\nu) = \min_{\pi} \pi\left (\frac{1}{2}\| \bm{x} - \bm{y}\|^2 \right ) + \varepsilon \mathsf{KL}(\pi || \mu \otimes \nu),
\end{align}
with $\mathsf{KL}$ being the Kullback-Liebler divergence
\begin{equation*}
    \mathsf{KL}(\pi || \mu \otimes \nu) = \int \log \left ( \frac{d\pi(\bm{x},\bm{y})}{d\mu(\bm{x})\cdot d\nu(\bm{y})} \right ) d\pi(\bm{x},\bm{y})
\end{equation*}
where the optimization is performed over $\Pi(\mu,\nu)$ -- the set of couplings of $\mu$ and $\nu$. These are probability measures over $\mathbb{R}^d\times \mathbb{R}^d$ such that
$\pi[A \times \mathbb{R}^d] = \mu[A]$ and $ \pi[\mathbb{R}^d \times A] = \nu[A]$ for every Borel measurable set $A$. 

The optimization \eqref{eq:EOT} also has a dual formulation which is given by
\begin{align*}
    \begin{aligned}
        S_\varepsilon(\mu,\nu) = \max_{f ,g} \mu(f) + \nu(g) - \varepsilon (\mu \otimes \nu)\left (\expe \left ( f + g - \frac{1}{2} \| \bm{x} - \bm{y}\|^2 \right ) \right ) + \varepsilon,
    \end{aligned}
\end{align*}
where the optimization is over all $f \in L_1(\mu), g \in L_1(\nu)$.
We will frequently refer to the pair $(f_\varepsilon,g_\varepsilon)$ realizing the maximum as the (optimal) dual potentials. The optimal coupling $\pi_\varepsilon$ is unique and the dual potentials $(f_\varepsilon,g_\varepsilon)$ are unique up to almost everywhere equivalence and shifting by a constant (\i.e., the pair $(f_\varepsilon + c, g_\varepsilon - c)$ is also optimal for any constant $c$) as long as $(\mu \otimes \nu)\left ( \frac{1}{2}\norm{\bm{x} - \bm{y}}^2 \right ) < \infty$ (\cite{nutz2021introduction} Theorem 4.2). In addition, $(f_\varepsilon,g_\varepsilon)$ is related to $\pi_\varepsilon$ by the following formula:
\begin{align} \label{eq:opt_relation}
    \frac{d \pi_\varepsilon(\bm{x},\bm{y})}{d \mu(\bm{x}) \otimes d\nu(\bm{y})} = \expe \left ( f_\varepsilon(\bm{x}) + g_\varepsilon(\bm{y}) - \frac{1}{2} \| \bm{x} - \bm{y}\|^2 \right).
\end{align}
We denote the right hand side by $p_\varepsilon(\bm{x},\bm{y})$ and refer to it as the \textit{relative density} of $\pi_\varepsilon$ with respect to $d\mu(\bm{x}) \otimes d\nu(\bm{y})$.
As a consequence, the optimal dual potentials can be expressed as
\begin{align*}
    f_\varepsilon(\bm{x}) &= -\varepsilon \log  \nu \left (\expe\left (g_\varepsilon(\bm{y}) - \frac{1}{2}\|\bm{x}-\bm{y}\|^2\right ) \right ) \\
    g_\varepsilon(\bm{y}) &= -\varepsilon \log  \mu \left (\expe \left (f_\varepsilon(\bm{x}) - \frac{1}{2}\|\bm{y}-\bm{x}\|^2 \right ) \right )
\end{align*}
and the optimal dual potentials can be taken so that equality holds for every $\bm{x}$ and $\bm{y}$ \cite{mena2019statistical}. 

Given these preliminaries, the main object of study in this paper is the entropy-regularized map, introduced in \cite{pooladian2021entropic} for approximating the unregularized OT map, defined as follows:
\begin{definition}
    Let $\mu,\nu \in \mathcal{P}_{2}(\mathbb{R}^d)$. The \textit{entropy-regularized map} from $\mu$ to $\nu$ is defined by
    \begin{equation*}
        T_\varepsilon(\bm{x}) = \mathbb{E}_{\bm{Y} \sim \pi_\varepsilon(\cdot|\bm{x})}[\bm{Y}] = \nu\left (\bm{y}p_\varepsilon(\bm{x},\bm{y}) \right ).
    \end{equation*}
\end{definition}
As pointed out in the introduction, this map has been studied and utilized in several recent works \cite{pooladian2021entropic,rigollet2022sample,MWMA_JMLR23,stromme23_ICML}. 

\subsection{Estimators of EOT Maps}

In practice all of the objects in the previous section must be estimated from a finite collection of samples. Given a pair of measures $\mu, \nu$ one can draw a finite set of samples $\bm{X}_1,...,\bm{X}_n \overset{i.i.d}{\sim} \mu, \bm{Y}_1,...,\bm{Y}_n \overset{i.i.d}{\sim} \nu$. We use the shorthands $\cX,\cY$ for the random sets $\{\bm{X}_1,...,\bm{X}_n\}$ and $\{\bm{Y}_1,...,\bm{Y}_n\}$ (allowing multiplicities) and define the empirical measures
\begin{equation*}
    \mu_{\cX} = \frac{1}{n}\sum_{i=1}^n \delta_{\bm{X}_i}, \hspace{1cm} \nu_{\cY} = \frac{1}{n}\sum_{j=1}^n\delta_{\bm{Y}_j}, \hspace{1cm} 
\end{equation*}
where $\delta_{\bm{z}}$ is the Dirac measure at the point $\bm{z}$.

When working with discrete measures we write $S_\varepsilon(\cX,\cY) = S_\varepsilon(\mu_{\cX},\nu_{\cY})$ and use $\pi_n, (f_n,g_n),$ and $ p_n$ to denote the entropy-regularized optimal coupling, dual potentials, and relative densities respectively. One important difference when working with empirical measures is that we have $d\mu_{\cX}(\bm{X}_i) \otimes d\nu_{\cY}(\bm{Y}_{j}) = \frac{1}{n^2}$ which gives the relation $\pi_n(\bm{X}_i,\bm{Y}_{j}) = \frac{1}{n^2}p_n(\bm{X}_i,\bm{Y}_j).$

We can also define the sample EOT map
\begin{equation*}
    T_n(\bm{X}_i) = \mathbb{E}_{\bm{Y} \sim \pi_n(\cdot|\bm{X}_i)}[\bm{Y}] = \sum_{j=1}^n \frac{1}{n} p_n(\bm{X}_i,\bm{Y}_j)\bm{Y}_j. \label{eq:Tn_formula}
\end{equation*}
Note that the functions $f_n, g_n, p_n$ and $T_n$ are only defined on $\cX, \cY,\cX \times \cY,$ and $\cX$ respectively. 
The potentials $f_n$ and $g_n$ can be \emph{extended} to all of $\mathbb{R}^d$ as two functions $f_n^{(e)}, g_n^{(e)}$ that satisfy the following two relations:
\begin{align*}
    \expe(f_n^{(e)}(\bm{x})) & = \frac{1}{n} \sum_{j=1}^n \expe  \left ( g_n^{(e)}(\bm{Y}_j) - \frac{1}{2} \| \bm{x} - \bm{Y}_j\|^2 \right )  \\
    \expe(g_n^{(e)}(\bm{y})) & = \frac{1}{n} \sum_{i=1}^n \expe \left ( f_n^{(e)}(\bm{X}_i) - \frac{1}{2} \| \bm{X}_i - \bm{y}\|^2 \right ), 
\end{align*}
where $g_n, g_n^{(e)}$ and $f_n, f_n^{(e)}$ agree with each other on the samples $\cY$ and $\cX$, respectively \cite{mena2019statistical}. This allows one to define the extended relative density 
\begin{equation*}
    p_n^{(e)}(\bm{x},\bm{y}) = \expe \left ( f_n^{(e)}(\bm{x}) + g_n^{(e)}(\bm{y}) - \frac{1}{2}\norm{\bm{x}-\bm{y}}^2 \right ).
\end{equation*}
Since $g_n, g_n^{(e)}$ and $f_n, f_n^{(e)}$ agree with each other on the samples it follows that the extended relative density satisfies $p_n^{(e)}(\bm{X}_i,\bm{Y}_j) = p_n(\bm{X}_i,\bm{Y}_j)$. Based on this notion an extended sample entropy-regularized map is proposed in \cite{pooladian2021entropic} via, 

\begin{align*}
    T_n^{(e)} (\bm{x}) = \sum_{j=1}^n \frac{1}{n} p_n^{(e)}(\bm{x},\bm{Y}_j)\bm{Y}_j.
\end{align*}
As pointed to in the introduction, a question which has garnered particular interest is to quantify how well $T_n^{(e)}$ approximates $T_\varepsilon$, especially when the measures are not compactly supported.

Existing results have only been able to show that $T_n^{(e)}$ converges to $T_\varepsilon$ at a rate of $O(e^{-1/\varepsilon}n^{-1})$ with respect to $\norm{\cdot}_{L^2(\mu)}^2$  \cite{rigollet2022sample} when the measures are compactly supported, and $O(e^{-1/\varepsilon^2}n^{-1/2})$ when the source is subGaussian and the target is compactly supported \cite{MWMA_JMLR23}. In order to alleviate the poor scaling in $\varepsilon$, an alternative estimator of $T_\varepsilon$ that does not require an extension is given in \cite{stromme23_ICML}. Given a point $\bm{x}$ at which an estimate of the map is sought and samples $\cX, \cY$, define an augmented set of samples ${\cX}^{(1)} = \{ \bm{x}, \bm{X}_2, \cdots, \bm{X}_n \}$, (i.e., the first sample in $\cX$ is replaced by the deterministic point $\bm{x}$). Let $p_n^{(1)}$ be the corresponding relative density. Then the estimate is defined via:
\begin{align}
\label{eq:SampEOTmap1}
    T_n^{(1)}(\bm{x}) = \sum_{j=1}^{n} \frac{1}{n}  p_n^{(1)}(\bm{x},\bm{Y}_j) \bm{Y}_j. 
\end{align}
The main idea here is that \emph{no extension is required} since the point at which the map needs to be estimated is in the sample set. 

They also propose in \cite{stromme23_ICML} to divide the samples $\cX, \cY$ into $k$ disjoint batches each containing $m = \frac{n}{k}$ samples (assume $k,m,n$ are all integers). For each batch $ \ell$ let the map $T^{(1)}_{\ell,m}(\bm{x})$ be the estimate in (\ref{eq:SampEOTmap1}) computed on $m$ samples. The final estimate is computed as an average over these batches:
\begin{equation}
    \hat{T}_{n}(\bm{x}) = \frac{1}{k} \sum_{\ell = 1}^{k} T_{\ell,m}^{(1)}(\bm{x}) \label{eq:T_hat}.
\end{equation}

In \cite{stromme23_ICML} $\hat{T}_{n}$ is analyzed in the case where $\mu$ and $\nu$ have compact support and they obtain essentially the same result as Theorem \ref{thm:ultimate_bounded} below.

\section{Analysis Overview} \label{sec:proof_structures}

Our analysis is done in essentially three steps. The first can be seen as a decomposition into a ``bias" term and a ``variance" term. It is an application of Jensen's inequality which actually holds for general random vectors.
\begin{lemma} \label{lem:triangle}
    For any measures $\mu,\nu$ we have 
    \begin{align*}
        \mathbb{E}_{\cX,\cY} [\|\hat{T}_n - T_\varepsilon \|^2_{L^2(\mu)}]  \leq 2\mathbb{E}_{\cX,\cY} [\|\hat{T}_n - \mathbb{E}[\hat{T}_n]\|_{L^2(\mu)}^2]  
         +2 \|\mathbb{E}[\hat{T}_n] - T_\varepsilon\|_{L^2(\mu)}^2.
    \end{align*}
\end{lemma}
The proof of this result and most of the technical results are deferred to the supplement. Here the ``bias" term is $\|\mathbb{E}[\hat{T}_n] - T_\varepsilon\|_{L^2(\mu)}^2$ which expresses the difference in the expectation and the function $T_\varepsilon$ we are trying to estimate. The ``variance" term is $\mathbb{E}_{\cX,\cY} [\|\hat{T}_n - \mathbb{E}[\hat{T}_n]\|_{L^2(\mu)}^2]$ which expresses how much that estimate $\hat{T}_n$ deviates from its expectation.
 
The two quantities are handled in quite different ways. It turns out that the variance term is easier to analyze and is done so by the following result.
\begin{lemma} \label{lem:easy_term} 
    Let $\bm{Y} \sim \nu$.  Then 
    \begin{equation*}
        \mathbb{E}_{\cX,\cY} [\|\hat{T}_n - \mathbb{E}[\hat{T}_n]\|_{L^2(\mu)}^2] \leq \mathbb{E}[\norm{\bm{Y} - \mathbb{E}\bm{Y}}^2]/k.
    \end{equation*}
    In particular if $\norm{\bm{Y}}_{\psi_2} < \infty$, then
    \begin{equation*}
        \mathbb{E}_{\cX,\cY} [\|\hat{T}_n - \mathbb{E}[\hat{T}_n]\|_{L^2(\mu)}^2] \leq 2\log(2)d\norm{\bm{Y}}_{\psi_2}^2/k.
    \end{equation*}
\end{lemma}

The bias term is more complicated. The general strategy is to define a measure $\tilde{\pi}$ which satisfies the following property. For every measurable set $A \subset \mathbb{R}^d \times \mathbb{R}^d$ it holds
\begin{equation*}
    \tilde{\pi}[A] = \mathbb{E}_{\cX,\cY}\left [ \pi_m[A] \right ] = (\mathbb{E}_{\cX,\cY}\left [ \pi_m \right ])[A].
\end{equation*}
The evaluation of $\pi_m[A]$ can be done for fixed samples $\cX,\cY$ via the following formula:
\begin{equation*}
    \pi_m[A] = \frac{1}{m^2}\sum_{i,j=1}^m p_m(\bm{X}_i,\bm{Y}_j)\cdot \pmb{1}[(\bm{X}_i,\bm{Y}_j) \in A].
\end{equation*}
The measure $\tilde{\pi}$ can be considered the average sample entropy-regularized coupling.  The measure $\tilde{\pi}$ depends on the choice of $m$, but since all the results involving $\tilde{\pi}$ hold for any $m \in \mathbb{N}$ we exclude it from the notation. Also note that we use $m$ instead of $n$ to define $\tilde{\pi}$ because we are interested in an estimate from a single batch of size $m$ instead of a sum over $k$ batches (which would use $n$ samples). This measure has several properties of note.  
\begin{proposition} \label{prop:pi_tilde_properties}
    The measure $\tilde{\pi}$ has density $d\tilde{\pi}(\bm{x},\bm{y})$ given by
    \begin{equation}
        \mathbb{E}_{\cX,\cY}\left [ p_m(\bm{X}_1,\bm{Y}_1) \big | \bm{X}_1 = \bm{x}, \bm{Y}_1 = \bm{y} \bm\right ] d\mu(\bm{x})d\nu(\bm{y}). \label{eq:intuitive_meaning} 
    \end{equation}
    In addition $\tilde{\pi}$ has marginals $\mu$ and $\nu$. As a consequence, $\tilde{\pi}$ has conditional density $d\tilde{\pi}(\bm{y}|\bm{x})$ given by
    \begin{equation*}
        \mathbb{E}_{\cX,\cY}\left [p_m(\bm{X}_1,\bm{Y}_1) \ | \ \bm{X}_1 = \bm{x}, \bm{Y}_1 = \bm{y} \right ] d\nu(\bm{y}).
    \end{equation*}
\end{proposition}

From \eqref{eq:intuitive_meaning} we can see that the interpretation of $d\tilde{\pi}(\bm{x},\bm{y})$ is that it is the average amount of mass the empirical coupling gives to the pair $(\bm{x},\bm{y})$, conditioned on both points being the first sample in $\cX$ and $\cY$ respectively.  

An important consequence of Proposition \ref{prop:pi_tilde_properties} is the following identity for $\hat{T}_n(\bm{x})$.
\begin{corollary} \label{cor:pi_tilde_cond_exp}
    For fixed $\bm{x}$ it holds
    \begin{equation*}
        \mathbb{E}[\hat{T}_n(\bm{x})] = \mathbb{E}_{\bm{Y} \sim \tilde{\pi}(\cdot|\bm{x})}[\bm{Y}].
    \end{equation*}
\end{corollary}
This provides the identity
\begin{equation*}
    \mathbb{E}[\hat{T}_n](\bm{x}) - T_{\varepsilon}(\bm{x}) = 
    \mathbb{E}_{\bm{Y} \sim \tilde{\pi}(\cdot|\bm{x})}[\bm{Y}] - \mathbb{E}_{\bm{Y} \sim \pi_\varepsilon(\cdot|\bm{x})}[\bm{Y}].
\end{equation*}
In several settings discussed below, using transport inequalities \cite{gozlan2010transport} we are able to control the deviation on the right using a KL-Divergence bound:
\begin{align} 
    \|\mathbb{E}_{\bm{Y} \sim \tilde{\pi}(\cdot|\bm{x})}[\bm{Y}] - \mathbb{E}_{\bm{Y} \sim \pi_\varepsilon(\cdot|\bm{x})}[\bm{Y}]\|^2 \lesssim \mathsf{KL}(\tilde{\pi}(\cdot|\bm{x}) \ || \  \pi_\varepsilon(\cdot|\bm{x})).
    \label{eq:KL_bound}
\end{align}
The final step is to bound the KL-divergence, which is handled by the following two results:
\begin{proposition} \label{prop:kl_bound}
    With the notation above, we have for any measures $\mu,\nu$:
    \begin{align*}
        \mathsf{KL}(\tilde{\pi} \ || \ \pi_\varepsilon) &= \mu(\mathsf{KL}(\tilde{\pi}(\cdot|\bm{X}) \ || \  \pi_\varepsilon(\cdot|\bm{X}))) \\
        &\leq \frac{1}{\varepsilon}\mathbb{E}_{\cX,\cY}[|S_\varepsilon(\mu,\nu) - S_\varepsilon(\cX,\cY)|].
    \end{align*}
\end{proposition}
This result is Lemma 8 in \cite{stromme23_ICML}, and we include a proof using our notation in the supplement. The final result that we will require is the following.
\begin{theorem} \label{thm:mena_weed} (\cite{mena2019statistical}) 
    Let $\mu,\nu$ be $\sigma^2$-norm-subGaussian. Then 
    \begin{equation*}
        \mathbb{E}_{\cX,\cY}[|S_\varepsilon(\mu,\nu) - S_\varepsilon(\cX,\cY)|] \lesssim \varepsilon \left ( 1 + \frac{\sigma^{\lceil 5d/2 \rceil + 6}}{\varepsilon^{\lceil 5d/4\rceil + 3}}\right ) \cdot \frac{1}{\sqrt{n}}
    \end{equation*}
\end{theorem}
The general proof strategy is now possible by applying Lemma \ref{lem:triangle} to break the error into two terms, one of which is controlled by Lemma \ref{lem:easy_term}. The other term is handled by the tandem of Corollary \ref{cor:pi_tilde_cond_exp}, Proposition \ref{prop:kl_bound}, and Theorem \ref{thm:mena_weed}. Therefore, what remains is to justify the bound in \eqref{eq:KL_bound}. In Section \ref{sec:bounded_slc}, we discuss two cases where this bound can be immediately arrived at.  In Section \ref{sec:subGaussian_measures}, we modify this approach slightly to analyze general subGaussian $\mu$ and $\nu$.

\section{Compactly-Supported and Strongly Log-Concave Target Measures}\label{sec:bounded_slc}

As noted above, the procedure above is applicable as long as we are able to justify a KL-Divergence upper bound. In this section we cover two cases where this can be done using known results: compactly supported and strongly log-concave measures. These require the following definition.
\begin{definition}
    Let $\mu \in \mathcal{P}(\mathbb{R}^d)$. The \textit{Laplace functional} of $\mu$ is given by
    \begin{equation*}
        E_\mu(\lambda) = \sup_f \mathbb{E}_{\bm{X} \sim \mu}[e^{\lambda f(\bm{X})}]
    \end{equation*}
    where the supremum is over all functions $f:\mathbb{R}^d \rightarrow \mathbb{R}$ which are 1-Lipschitz with respect to the Euclidean distance and satisfy $\mu(f) = 0$.
\end{definition}
The following T1-transport inequality is classical.
\begin{theorem} \label{thm:w1_kl}(\cite{bobkov1999exponential, ledoux2001concentration})
    Let $\mu \in \mathcal{P}(\mathbb{R}^d)$. Then 
    \begin{equation*}
        W_1(\mu,\nu) \leq \sqrt{2C\mathsf{KL}(\nu || \mu)}
    \end{equation*}
    for some $C > 0$ and all $\nu$ if and only if 
    \begin{equation*}
        E_{\mu}(\lambda) \leq e^{C\lambda^2/2}, \hspace{1cm} \forall \ \lambda \geq 0.
    \end{equation*}
\end{theorem}
This theorem is useful in our setting because of the inequality
\begin{align}
    \|\mathbb{E}_{\bm{Y} \sim \tilde{\pi}(\cdot|\bm{x})}[\bm{Y}] - \mathbb{E}_{\bm{Y} \sim \pi_\varepsilon(\cdot|\bm{x})}[\bm{Y}]\|^2 \leq W_1^2(\tilde{\pi}(\cdot|\bm{x}), \pi_\varepsilon(\cdot|\bm{x})) \label{eq:expectation_vs_w1}
\end{align}
which follows from taking a $W_1$-optimal coupling of $\tilde{\pi}(\cdot|\bm{x})$ and $\pi_\varepsilon(\cdot|\bm{x})$ followed by Jensen's inequality. Therefore, whenever we can apply Theorem \ref{thm:w1_kl} we have
\begin{align*}
    \|\mathbb{E}_{\bm{Y} \sim \tilde{\pi}(\cdot|\bm{x})}[\bm{Y}] - \mathbb{E}_{\bm{Y} \sim \pi_\varepsilon(\cdot|\bm{x})}[\bm{Y}]\|^2 \leq 2C\mathsf{KL}(\tilde{\pi}(\cdot|\bm{x}) \ || \ \pi_\varepsilon(\cdot|\bm{x})) 
\end{align*}
which is the key step that we required in the previous section. We have now turned our task to showing that the Laplace functional of $\pi_\varepsilon(\cdot|\bm{x})$ satisfies the bound in Theorem \ref{thm:w1_kl}. 

\subsection{Compactly-Supported Measures}

The first setting where bounds are readily available for the Laplace functional is in the case of compactly-supported measures.
\begin{proposition} \label{prop:lap_bounded} (\cite{ledoux2001concentration} Proposition 1.16, Adapted) Let $\mu \in \mathcal{P}(\mathbb{R}^d)$ with support contained in $B(0,R)$. Then it holds that 
\begin{equation*}
    E_{\mu}(\lambda) \leq e^{4R^2\lambda^2/2}.
\end{equation*}
\end{proposition}
Combining Theorem \ref{thm:w1_kl} with Proposition \ref{prop:lap_bounded} we have the following result.
\begin{theorem}  \label{thm:bounded_T1}
    Let $\nu$ be a probability measure with support contained in $B(0,R)$. Then for every $\bm{x}$ we have
    \begin{equation*}
        W_1^2(\pi_\varepsilon(\cdot|\bm{x}), \tilde{\pi}(\cdot|\bm{x})) \leq 8R \cdot \mathsf{KL}(\tilde{\pi}(\cdot|\bm{x}) \ || \ \pi_\varepsilon(\cdot|\bm{x}) ).
    \end{equation*}
\end{theorem}

Combining the arguments of this section and the previous we have the following result.

\begin{theorem} \label{thm:ultimate_bounded}
    Let $\mu$ be $\sigma^2$-norm-subGaussian and let $\nu$ be a probability measure with support contained in $B(0,R)$ for $R < \infty$. Let $\sigma_0^2 = \max(\sigma^2, R^2/(2d\log2))$ so that $\mu,\nu$ are both $\sigma_0^2$-norm-subGaussian. Then 
    \begin{equation*}
        \mathbb{E}_{\cX,\cY}\left[ \norm{\hat{T}_n - T_\varepsilon}^2_{L^2(\mu)} \right ] \lesssim \frac{R^2}{k} + R\left ( 1 + \frac{\sigma_0^{\lceil 5d/2 \rceil + 6}}{\varepsilon^{\lceil 5d/4\rceil + 3}} \right ) \frac{1}{\sqrt{m}}
    \end{equation*}
    In particular if $k = n^{1/3}, m = n^{2/3}$ 
    \begin{equation*}
        \mathbb{E}_{\cX,\cY}\left[ \norm{\hat{T}_n - T_\varepsilon}^2_{L^2(\mu)} \right ] \lesssim R\left (1 + R + \frac{\sigma_0^{\lceil 5d/2 \rceil + 6}}{\varepsilon^{\lceil 5d/4\rceil + 3}} \right ) \frac{1}{n^{1/3}}.
    \end{equation*}
\end{theorem}

\subsection{Strongly Log-Concave Measures}

Another setting where Theorem \ref{thm:w1_kl} can be applied is when the measure $\nu$ is strongly log-concave. 
\begin{definition}
    A probability measure $\nu$ is said to be $c$-\textit{strongly log-concave} if it has density with respect to the Lebesgue measure of the form $q(\bm{y}) \propto e^{-W(\bm{y})}$ for a $c$-strongly convex function $W$. That is,
    \begin{equation*}
        \nabla^2 W(\bm{y}) \succeq c\text{I}
    \end{equation*}
    for every $\bm{y}$ in the domain of $W$.
\end{definition}
Strongly log-concave measures satisfy the following bound, which is essentially contained in \cite{ledoux2001concentration}.
\begin{theorem} \label{thm:strongly_lc_lap}
    Let $\nu$ be $c$-strongly log-concave. Then the Laplace functional satisfies
    \begin{equation*}
        E_{\nu}(\lambda) \leq e^{\frac{1}{c}\lambda^2/2}.
    \end{equation*}
\end{theorem}
For completeness, a proof is given in the supplement and is based on the relationship between strong log-concavity, log-Sobolev inequalities, and the Laplace functional. The next result allows us to apply Theorem \ref{thm:strongly_lc_lap} to the family of measures $\pi_\varepsilon(\cdot|\bm{x})$. 
\begin{lemma} \label{lem:conditionals_strongly_lc}
    Let $\nu$ be $c$-strongly log-concave. Then for every $\bm{x}$ the measure $\pi_\varepsilon(\cdot|\bm{x})$ is also $c$-strongly log-concave.
\end{lemma}
This result follows from \eqref{eq:opt_relation} and the fact that the function $\frac{1}{2}\norm{\cdot}^2 - g_\varepsilon$ is convex. Importantly the constant strong log-concavity parameter $c$ is independent of the point $\bm{x}$.

Combining Lemma \ref{lem:conditionals_strongly_lc}, Theorem \ref{thm:strongly_lc_lap}, and Theorem \ref{thm:w1_kl} we obtain the following bound.
\begin{theorem}
    Let $\nu$ be $c$-strongly log-concave. Then for every $\bm{x}$ we have
    \begin{equation*}
        W_1^2(\pi_\varepsilon(\cdot|\bm{x}), \tilde{\pi}(\cdot|\bm{x})) \leq \frac{2}{c}\mathsf{KL}(\pi_\varepsilon(\cdot|\bm{x}) \ || \ \tilde{\pi}(\cdot|\bm{x})).
    \end{equation*}
\end{theorem}
The following result is now immediate from the approach outlined in Section \ref{sec:proof_structures}.
\begin{theorem} \label{thm:ultimate_lcc}
    Let $\mu$ be $\sigma^2$-norm-subGaussian and let $\nu$ be $c$-strongly log-concave and mean zero. Let $\sigma_0 = \max(\sigma^2, K/\sqrt{c})$, where $K > 0$ is a universal constant, so that $\mu$ and $\nu$ are $\sigma_0^2$-norm-subGaussian. Then  
    \begin{equation*}
        \mathbb{E}_{\cX,\cY}\left[ \norm{\hat{T}_n - T_\varepsilon}^2_{L^2(\mu)} \right ] \lesssim \frac{d}{ck} + \frac{1}{c}\left ( 1 + \frac{\sigma_0^{\lceil 5d/2 \rceil + 6}}{\varepsilon^{\lceil 5d/4\rceil + 3}} \right ) \frac{1}{\sqrt{m}}.
    \end{equation*}
    In particular if $k = n^{1/3}, m = n^{2/3}$ 
    \begin{equation*}
        \mathbb{E}_{\cX,\cY}\left[ \norm{\hat{T}_n - T_\varepsilon}^2_{L^2(\mu)} \right ] \lesssim \frac{1}{c}\left ( d + \frac{\sigma_0^{\lceil 5d/2 \rceil + 6}}{\varepsilon^{\lceil 5d/4\rceil + 3}}\right ) \cdot \frac{1}{n^{1/3}}
    \end{equation*}
\end{theorem}

\section{General SubGaussian Measures} \label{sec:subGaussian_measures}

We start with a result on the Laplace functional of norm-subGaussian random variables. 

\begin{lemma} \label{lem:subg_lap_bound}
    Let $\bm{X} \sim \mu$ be such that $\norm{\bm{X}}_{\psi_2}^2 < \infty$. Then
    \begin{equation*}
        E_\mu(\lambda) \leq e^{8d\norm{\bm{X}}_{\psi_2}^2\lambda^2}.
    \end{equation*}
\end{lemma}

Next we will prove a utility result on subGaussian random variables; we state it in generality and specialize to our case.
\begin{lemma} \label{lem:expected_squared_concentration}
    Let $\mu,\nu$ be probability measures on $\mathbb{R}^d$ and let $\pi$ be a probability measure on $\mathbb{R}^d \times \mathbb{R}^d$ with marginals $\mu$ and $\nu$. Let $\bm{X} \sim \mu, \bm{Y} \sim \nu,$ and let $\bm{Y}^{\bm{x}}$ be distributed according to $\pi(\cdot|\bm{x})$. Then it holds
    \begin{equation*}
        \mathbb{E}_{\bm{X}} \left [ \norm{\bm{Y}^{\bm{X}}}_{\psi_2}^2 \right ] \leq 2\norm{\bm{Y}}_{\psi_2}^2
    \end{equation*}
\end{lemma}
This result essentially says that most of the conditional measures of a joint distribution are at least as well concentrated as the marginal. This is a key technical tool as will be seen in the proof of the following result.
\begin{theorem} \label{thm:ultimate_subg}
    Let $\mu,\nu$ be $\sigma^2$-norm-subGaussian. Then 
    \begin{align*}
        \mathbb{E}\left [ \norm{\hat{T}_n - T_\varepsilon}_{L^1(\mu)} \right ] \lesssim 
        \frac{\sqrt{d\sigma^2}}{\sqrt{k}} + \sqrt{d\sigma^2\left ( 1 + \frac{\sigma^{\lceil 5d/2 \rceil + 6}}{\varepsilon^{\lceil 5d/4\rceil + 3}} \right )} \frac{1}{m^{1/4}}.
    \end{align*}
    In particular if $k = n^{1/3},m=n^{2/3}$ then
    \begin{align*}
        \mathbb{E}\left [ \norm{\hat{T}_n - T_\varepsilon}_{L^1(\mu)} \right ] \lesssim \sqrt{
        d\sigma^2\left ( 1 + \frac{\sigma^{\lceil 5d/2 \rceil + 6}}{\varepsilon^{\lceil 5d/4\rceil + 3}} \right )} \frac{1}{n^{1/6}}.
    \end{align*}
\end{theorem}

\begin{proof}
    We start in a similar way to the previous proofs:
    \begin{align*}
        \mathbb{E} \norm{\hat{T}_n - T_\varepsilon}_{L^1(\mu)} \leq  \mathbb{E} \norm{\hat{T}_n - \mathbb{E}[\hat{T}_n]}_{L^1(\mu)} +  \norm{\mathbb{E}[\hat{T}_n] - T_\varepsilon}_{L^1(\mu)}.
    \end{align*}
    The first term is handled by Jensen's inequality:
    \begin{align*}
        \mathbb{E} \norm{\hat{T}_n - \mathbb{E}[\hat{T}_n]}_{L^1(\mu)} \leq \mathbb{E} \left [ \norm{\hat{T}_n - \mathbb{E}[\hat{T}_n]}_{L^2(\mu)}^2 \right ]^{1/2}
    \end{align*}
    and the resulting term can be controlled by Lemma \ref{lem:easy_term}.

    For the other term, we have:
    \begin{align*}
        \mathbb{E}_{\bm{X}} \left [  \norm{\mathbb{E}[\hat{T}_n](\bm{X}) - T_\varepsilon(\bm{X})} \right ] \leq \mathbb{E}_{\bm{X}} \left [ W_1(\tilde{\pi}(\cdot|\bm{X}), \pi_\varepsilon(\cdot | \bm{X})) \right ].
    \end{align*}
    The strategy now is to apply Theorem \ref{thm:w1_kl} conditionally with a constant that depends on the value of $\bm{X}$. In particular if we let $\bm{Y}^{\bm{x}} \sim \tilde{\pi}(\cdot|\bm{x})$ then we have by Lemma \ref{lem:subg_lap_bound} and Theorem \ref{thm:w1_kl} that:
    \begin{align*}
        W_1(\tilde{\pi}(\cdot|\bm{x}), \pi_\varepsilon(\cdot | \bm{x})) \leq \sqrt{16d\norm{\bm{Y}^{\bm{x}}}^2_{\psi_2}\mathsf{KL}(\tilde{\pi}(\cdot|\bm{x}) \ || \  \pi_\varepsilon(\cdot | \bm{x}))}
    \end{align*}
    Using this inequality above and applying Cauchy-Schwarz we have:
    \begin{align*}
        \mathbb{E}_{\bm{X}} \left [ W_1(\tilde{\pi}(\cdot|\bm{X}), \pi_\varepsilon(\cdot | \bm{X})) \right ] &\leq \mathbb{E}_{\bm{X}} \left [ \sqrt{16d\norm{\bm{Y}^{\bm{X}}}^2_{\psi_2}\mathsf{KL}(\tilde{\pi}(\cdot|\bm{X}) \ || \  \pi_\varepsilon(\cdot | \bm{X}))} \right ] \\
        &\leq 4\sqrt{d}\mathbb{E}_{\bm{X}}\left [ \norm{\bm{Y}^{\bm{X}}}^2_{\psi_2}\right ]^{1/2} \cdot \mathbb{E}_{\bm{X}}\left [\mathsf{KL}(\tilde{\pi}(\cdot|\bm{X}) \ || \  \pi_\varepsilon(\cdot | \bm{X})) \right ]^{1/2}.
    \end{align*}
    The result now follows from Lemma \ref{lem:expected_squared_concentration}, Proposition \ref{prop:kl_bound}, and Theorem \ref{thm:mena_weed}.
\end{proof}

The main reason that we obtain a different bound in Theorem \ref{thm:ultimate_subg} compared to Theorems \ref{thm:ultimate_bounded} and \ref{thm:ultimate_lcc} is that we do not have a uniform control on the constant used when conditionally applying Theorem \ref{thm:w1_kl}. This is handled by Cauchy-Schwarz inequality which adds an additional square-root, leading to a bound of $n^{-1/6}$ instead of $n^{-1/3}$. We suspect that this may be avoidable (see Conjecture \ref{conj:subg}).

\section{Numerical Investigations} \label{sec:numerics}

All code used to generate the figures throughout this section will be made publicly available. Detailed descriptions of the parameters used and further numerical experiments are deferred to Section \ref{sec:numeric_details} in the supplement. 

In this section we consider a few different settings to empirically verify and shed more light on the results in Theorems \ref{thm:ultimate_bounded}, \ref{thm:ultimate_lcc}, and \ref{thm:ultimate_subg}. 

\subsection{Gaussian Results} 

\begin{figure}
    \centering
    \includegraphics[width=0.5\linewidth]{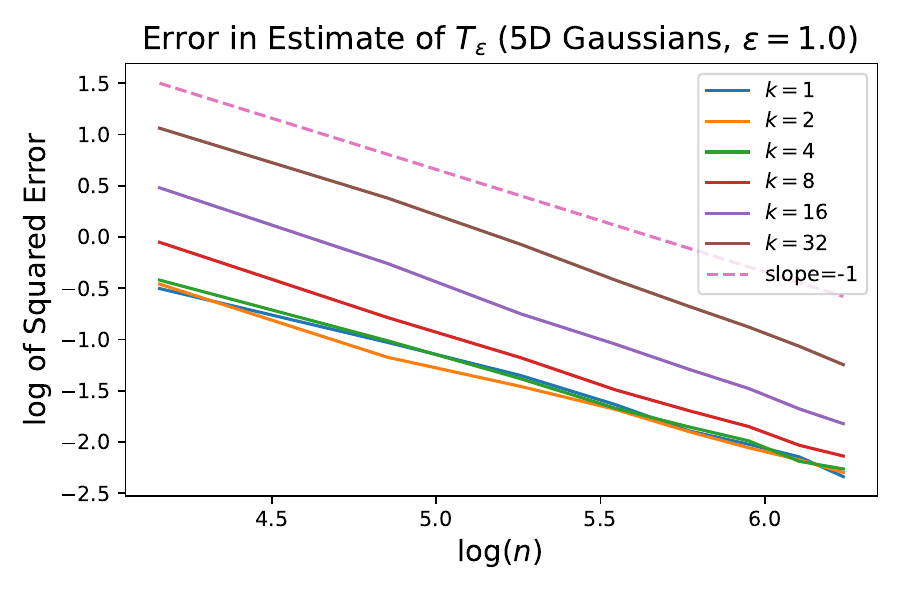}
    \caption{Convergence rate of $\hat{T}_n$ in squared $L_2(\mu)$, in log-scale for different choices of $k$. (Averaged over 100 replications)}
    \label{fig:exact_gaussian_eps1}
\end{figure}

In Figure \ref{fig:exact_gaussian_eps1} we show the log-log plot of estimation error vs $n$ for varying choices of $k$ when $\mu = N(\bm{x}_0, \Sigma_0), \nu = N(\bm{x}_1, \Sigma_1)$ and utilize the known closed-form expression for $T_\varepsilon$ from \cite{janati2020entropic}. For each $n$ we draw $n$ samples from $\mu$ and $\nu$ which are used to construct $\hat{T}_n$, and a further 500 independent samples $\bm{X}_1,...,\bm{X}_{500} \sim \mu$ which are used for the Monte-Carlo estimation 
\begin{equation} \label{eq:monte_carlo}
    \norm{\hat{T}_n - T_\varepsilon}_{L^2(\mu)}^2 \approx \frac{1}{n} \sum_{i=500}^\ell \norm{\hat{T}_n(X_i) - T_\varepsilon(X_i)}_2^2.
\end{equation}
We compare across several settings of the number of batches $k$ and compare the results. We observe the smaller the number of batches (equivalently the larger each batch is) the better the estimator is. Additionally, for all settings of $k$ we observe decay approximately proportional to $n^{-1}$. This suggests that our analysis may be loose, at least in the Gaussian case.

\subsection{Variance Estimates}
In this section we explore the possible dependence on $m$ of the variance term in the bias-variance decomposition analysis. To estimate this quantity without estimating $\mathbb{E}[\hat{T}_n]$ let $\hat{T}_n'$ be an independent map estimate from $\hat{T}_n$ (meaning that it is generated by a separate set of samples $\cX',\cY'$) and let $T_{1,m}^{(1)},{T_{1,m}^{(1)}}'$ be single terms in $\hat{T}_n$ and $\hat{T}_n'$ as in \eqref{eq:T_hat}. One can show that
\begin{align*}
    2 \mathbb{E}\norm{\hat{T}_n - \mathbb{E}[\hat{T}_n]}_{L^2(\mu)}^2 = \frac{1}{k}\mathbb{E}\norm{T_{1,m}^{(1)} - {T_{1,m}^{(1)}}'}_{L^2(\mu)}^2.
\end{align*} 
Note that the right hand side only relies on comparing two independent estimates which can be easily constructed.

\begin{figure}
    \centering
    \begin{minipage}{.45\textwidth}
        \centering
        \includegraphics[width=0.9\linewidth]{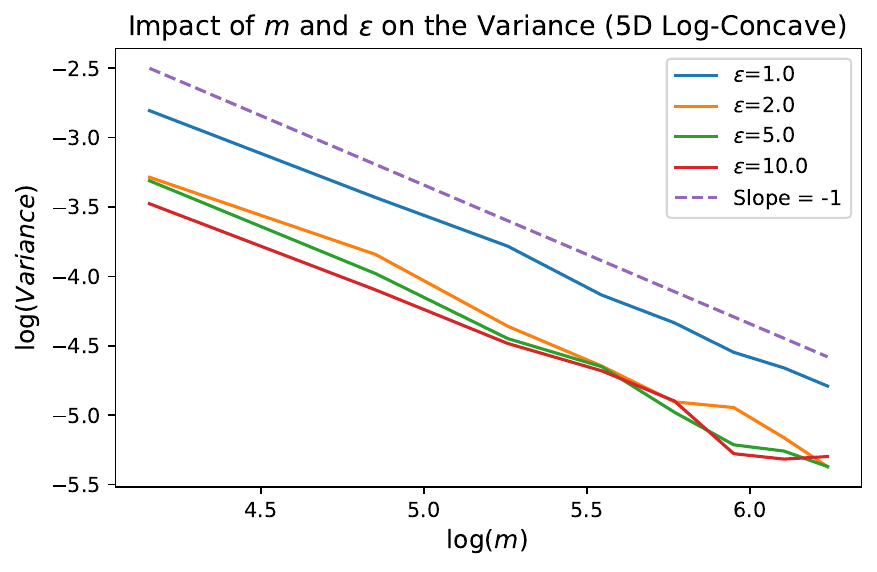}
        \captionof{figure}{Estimate of the variance term in the bias variance decomposition in log-scale for 5D log-concave measures.}
        \label{fig:variance_lcc}
    \end{minipage} %
    \hfill
    \begin{minipage}{.45\textwidth}
        \centering
        \includegraphics[width=0.9\linewidth]{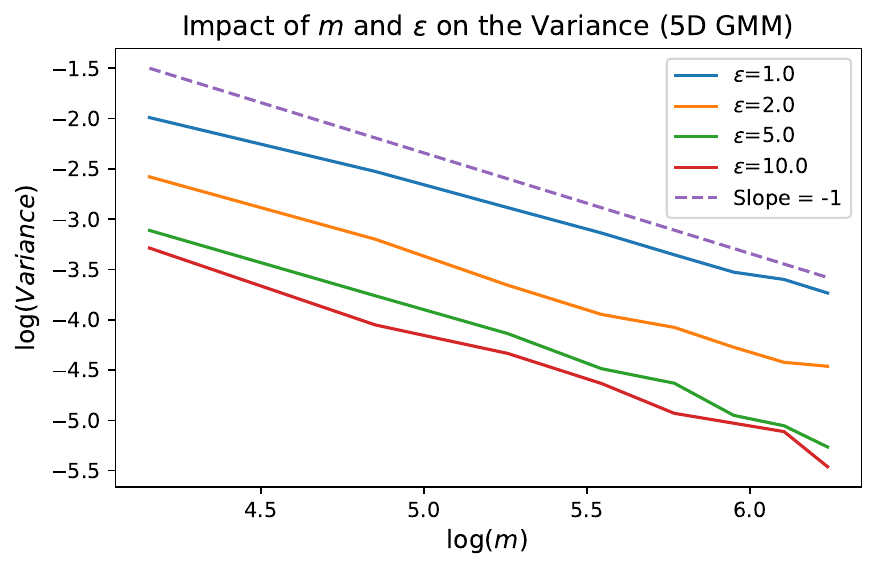}
        \captionof{figure}{Estimate of the variance term in the bias variance decomposition in log-scale for 5D Gaussian mixture models.}
        \label{fig:variance_gmm}
    \end{minipage}
\end{figure}

Figures \ref{fig:variance_lcc} and \ref{fig:variance_gmm} show the log-log plots of the variance term vs $m$ with varying $\varepsilon$ for two cases: when the source and targets are both strongly log-concave (but not multivariate Gaussians) and when the source and targets are both Gaussian Mixture Models, which models the non-log-concave but subGaussian case.

In both cases we use Monte Carlo integration as in \eqref{eq:monte_carlo} and use 100 replications. These figures demonstrate two things. First, increasing $\varepsilon$ tends to decrease the variance. Second, the variance appears to decay at a rate of roughly $n^{-1}$. Due to these empirical findings we propose the following:
\begin{conjecture} \label{conj:variance}
    Let $\mu,\nu$ be subGaussian. Then 
    \begin{equation*}
        \mathbb{E}\norm{T_{1,m} - \mathbb{E}[T_{1,m}]}_{L^2(\mu)}^2 \lesssim \frac{1}{m}
    \end{equation*}
    or equivalently
    \begin{equation*}
        \mathbb{E}\norm{\hat{T}_n - \mathbb{E}[\hat{T}_n]}_{L^2(\mu)}^2 \lesssim \frac{1}{k}\cdot \frac{1}{m} = \frac{1}{n}
    \end{equation*}
    with a constant that can be taken independently of $m,k,$ and $n$.
\end{conjecture}
Consequences of this conjecture are discussed in the next Section.

\section{Open Problems and Conclusion}

\paragraph{The Variance Conjecture.} The most impactful open question is stated in Conjecture \ref{conj:variance}, and is empirically observed in Figures \ref{fig:variance_lcc} and \ref{fig:variance_gmm}. A positive answer would provide a substantial strengthening of Lemma \ref{lem:easy_term} from convergence of the order $k^{-1}$ to $n^{-1}$. As a consequence the bounds in Theorems \ref{thm:ultimate_bounded} and \ref{thm:ultimate_lcc} can be improved from order $k^{-1} + m^{-1/2}$ to $n^{-1}+ m^{-1/2}$. Choosing $m=n$ gives the bound $n^{-1} + n^{-1/2}$ which in turn improves the convergence rate from $n^{-1/3}$ to $n^{-1/2}$. In addition, the result in Theorem \ref{thm:ultimate_subg} can be improved from $n^{-1/6}$ to $n^{-1/4}$ in a similar way.

\paragraph{Squaring in the subGaussian Case.} 
As discussed after the proof of Theorem \ref{thm:ultimate_subg}, conditional application of Theorem \ref{thm:w1_kl} appears to be a substantial technical hurdle to obtaining a version of Theorem \ref{thm:ultimate_subg} of the same type as Theorems \ref{thm:ultimate_bounded} and \ref{thm:ultimate_lcc}. However, we expect that the following result should be possible.
\begin{conjecture} \label{conj:subg}
    In the same setting as Theorem \ref{thm:ultimate_subg}, it holds that
    \begin{equation*}
        \mathbb{E} \left [ \norm{\hat{T}_n - T_\varepsilon}_{L^2(\mu)}^2 \right ] \lesssim n^{-1/3}
    \end{equation*}
    where the implicit constant may depend on any parameter besides $n$.
\end{conjecture}
It may be the case that obtaining such a result requires a different analysis altogether. Additionally, if Conjecture \ref{conj:variance} holds then this should be improvable to $n^{-1/2}$.

\paragraph{Convergence at a rate of $n^{-1}$.} It is shown in \cite{rigollet2022sample} that for compactly supported measures the entropy-regularized map converges at a rate of $n^{-1}$, even faster than the rate of $n^{-1/2}$ recovered in our proofs. At least for multivariate Gaussian source and target we empirically observe the same rate of convergence in Figure \ref{fig:exact_gaussian_eps1}. Obtaining this rate explicitly is likely an interesting avenue of study.
\begin{conjecture}
    Suppose that $\mu,\nu$ are mean-zero multivariate Gaussians. Then
    \begin{equation*}
        \mathbb{E} \left [ \norm{\hat{T}_n - T_\varepsilon}_{L^2(\mu)}^2 \right ] \lesssim \frac{1}{n}.
    \end{equation*}
\end{conjecture}
It is still open to identify in which settings one can expect this even faster rate to apply.

\paragraph{Conclusion} In this work we have established a framework for obtaining convergence rates for an estimate of the  EOT map, established convergence rates in three settings, two of which do not require compactness of the support, and posed three important open questions based on numerical investigations. Furthermore, it will be interesting to directly compare the extended and the in-sample estimators discussed in Section \ref{sec:problem_set_up}. While it is intuitive that they should be close to each other since they both converge to $T_\varepsilon$ with $n$, quantifying the similarity will shed more light on possible tradeoffs between the two estimators including the polynomial vs exponential dependence on $\varepsilon$ vs the rate in $n$. Finally, establishing high-probability versions of these results is also an important direction of study.

\section{Acknowledgements}
This research is supported by NSF DMS-2309519 and was supported in part by NSF CAREER 1553075 and the Camille and Henry Dreyfus Foundation. James M. Murphy acknowledges support in part by NSF DMS-2318894.  Shuchin Aeron acknowledges support in part by NSF DRL-1931978. 

\bibliographystyle{abbrv} 
\bibliography{references}

\appendix
\onecolumn

\section{Proofs From Section \ref{sec:proof_structures}}

One additional fact that we will use several times in this section is that since $\pi_n$ satisfies the marginal constraints the following holds for every $i,j=1,...,n$
\begin{align}
    \frac{1}{n} &= \sum_{i=1}^n \pi_n(\bm{X}_i,\bm{Y}_j) = \sum_{j=1}^n \pi_n(\bm{X}_i,\bm{Y}_j) \notag \\
    1 &= \sum_{i=1}^n \frac{1}{n}p_n(\bm{X}_i,\bm{Y}_j) = \sum_{j=1}^n \frac{1}{n}p_n(\bm{X}_i,\bm{Y}_j). \label{eq:pn_sum}
\end{align}

\subsection{Proof of Lemma \ref{lem:triangle}}

\begin{proof}
    Use the identity
    \begin{equation*}
        \hat{T}_n - T_\varepsilon = \frac{1}{2}\cdot 2(\hat{T}_n - \mathbb{E}[\hat{T}_n]) + \frac{1}{2}\cdot 2 (\mathbb{E}[\hat{T}_n] - T_\varepsilon)
    \end{equation*}
    then apply Jensen's inequality inside the expectation and simplify. The second term no longer has an expectation with respect to the samples $\cX,\cY$ because the two quantities are not random.
\end{proof}

\subsection{Proof of Lemma \ref{lem:easy_term}} \label{sec:easy_term} 

\begin{proof} 
    Throughout the proof the expectations are to be taken with respect to the random sets $\cX$ and $\cY$ unless otherwise stated.

    First recall the formula for the estimator \eqref{eq:T_hat},
    \begin{equation*}
        \hat{T}_n = \frac{1}{k} \sum_{\ell=1}^k T^{(1)}_{\ell,m}
    \end{equation*}
    and therefore
    \begin{equation*}
        \mathbb{E}[\hat{T}_n] = \frac{1}{k} \sum_{\ell=1}^k \mathbb{E}[T_{\ell,m}^{(1)}].
    \end{equation*}
    Since $\norm{\cdot}_{L^2(\mu)}$ is a Hilbert space we have
    \begin{align*}
        \norm{\hat{T}_n - \mathbb{E}[\hat{T}_n]}_{L^2(\mu)}^2 &= \left \langle \frac{1}{k} \sum_{\ell=1}^k \left (T_{\ell,m}^{(1)} - \mathbb{E}\left [T_{\ell,m}^{(1)} \right ] \right ), \frac{1}{k} \sum_{\ell'=1}^k \left (T_{\ell',m}^{(1)} - \mathbb{E}\left [T_{\ell',m} ^{(1)}\right ] \right ) \right \rangle_{L^2(\mu)} \\
        &= \frac{1}{k^2} \sum_{\ell=1}^k \norm{T_{\ell,m}^{(1)} - \mathbb{E}\left [T_{\ell,m}^{(1)} \right ]}_{L^2(\mu)}^2 + \frac{1}{k^2} \sum_{\ell\neq\ell'}^k \left \langle T_{\ell,m}^{(1)} - \mathbb{E}\left [T_{\ell,m}^{(1)} \right ], T_{\ell',m}^{(1)} - \mathbb{E}\left [T_{\ell',m}^{(1)} \right ] \right \rangle_{L^2(\mu)}.
    \end{align*}
    Taking expectations on both sides and using the fact that the $T_{\ell,m}^{(1)}$ are independent and identically distributed we have
    \begin{align*}
        &\mathbb{E}\norm{\hat{T}_n - \mathbb{E}[\hat{T}_n]}_{L^2(\mu)}^2 \\
        &= \frac{1}{k^2}\sum_{\ell=1}^k\mathbb{E}\norm{T_{\ell,m}^{(1)} - \mathbb{E}\left [T_{\ell,m}^{(1)}\right]}_{L^2(\mu)}^2 + \frac{1}{k^2} \sum_{\ell \neq \ell'} \mathbb{E}\left \langle T_{\ell,m}^{(1)} - \mathbb{E}\left [T_{\ell,m}^{(1)} \right ], T_{\ell',m}^{(1)} - \mathbb{E}\left [T_{\ell',m}^{(1)} \right ] \right \rangle_{L^2(\mu)} \\
        &= \frac{1}{k}\mathbb{E}\norm{T_{1,m}^{(1)} - \mathbb{E}\left [T_{1,m}^{(1)} \right ]}_{L^2(\mu)}^2 + \frac{k^2-k}{k^2} \left \langle \mathbb{E}\left [ T_{1,m}^{(1)} - \mathbb{E}\left [T_{1,m}^{(1)} \right ] \right ], \mathbb{E}\left [T_{2,m}^{(1)} - \mathbb{E}\left [T_{2,m}^{(1)}\right ] \right ] \right \rangle_{L^2(\mu)} \\
        &= \frac{1}{k}\mathbb{E}\norm{T_{1,m}^{(1)} - \mathbb{E}\left [T_{1,m}^{(1)}\right ]}_{L^2(\mu)}^2
    \end{align*}
    where we have used in the third line the independence and identical distribution to move to remove the summations and also bring the expectations inside the inner product.

    Now exchanging the order of integration we have
    \begin{equation*}
        \mathbb{E}\norm{T_{1,m}^{(1)} - \mathbb{E}\left [T_{1,m}^{(1)} \right ]}_{L^2(\mu)}^2 = \mathbb{E}_{\bm{X}} \left [ \mathbb{E} \left [ \norm{T_{1,m}^{(1)}(\bm{X}) - \mathbb{E}\left [T_{1,m}^{(1)}(\bm{X})\right ]}^2\right ] \right ].
    \end{equation*}
    Now observe that for every $i=1,...,m$ we have by exchangeability
    \begin{equation*}
        \mathbb{E}_{\bm{X}} \left [ \mathbb{E} \left [ \norm{T_{1,m}^{(1)}(\bm{X}) - \mathbb{E}\left [T_{1,m}^{(1)}(\bm{X})\right ]}^2\right ] \right ] = \mathbb{E}_{\bm{X}} \left [ \mathbb{E} \left [ \norm{T_{1,m}^{(i)}(\bm{X}) - \mathbb{E}\left [T_{1,m}^{(i)}(\bm{X})\right ]}^2\right ] \right ].
    \end{equation*}
    Importantly by the variational formula for variance we have the inequality for every fixed $\bm{x}$ and $i$
    \begin{equation*}
        \mathbb{E} \left [ \norm{T_{1,m}^{(i)}(\bm{x}) - \mathbb{E}\left [T_{1,m}^{(i)}(\bm{x})\right ]}^2\right ] \leq \mathbb{E} \left [ \norm{T_{1,m}^{(i)}(\bm{x}) - \mathbb{E}[\bm{Y}]}^2\right ]
    \end{equation*}
    From this it follows that
    \begin{equation*}
        \mathbb{E}_{\bm{X}} \left [ \mathbb{E} \left [ \norm{T_{1,m}^{(1)}(\bm{X}) - \mathbb{E}\left [T_{1,m}^{(1)}(\bm{X})\right ]}^2\right ] \right ] \leq \frac{1}{m}\sum_{i=1}^m \mathbb{E}_{\bm{X}} \left [ \mathbb{E} \left [ \norm{T_{1,m}^{(i)}(\bm{X}) - \mathbb{E}[\bm{Y}]}^2\right ] \right ] 
    \end{equation*}
    The key observation at this step is that $T_{1,m}^{(i)}(\bm{X})$ is a deterministic function of $(\bm{X}_1,...,\bm{X}_{i-1},\bm{X},\bm{X}_{i+1},...,\bm{X}_m,\cY)$ and $T_{1,m}(\bm{X}_i)$ is the same deterministic function of $(\bm{X}_1,...,\bm{X}_{i-1},\bm{X}_i,\bm{X}_{i+1},...,\bm{X}_m,\cY)$ and the random variables $(\bm{X}_1,...,\bm{X}_{i-1},\bm{X},\bm{X}_{i+1},...,\bm{X}_m,\cY)$ and $(\bm{X}_1,...,\bm{X}_{i-1},\bm{X}_i,\bm{X}_{i+1},...,\bm{X}_m,\cY)$ have the same distribution (both are distributed according to $\mu^{\otimes m}\otimes \nu^{\otimes m}$). From this we can conclude that
    \begin{equation*}
        \mathbb{E}_{\bm{X}}\left  [ \mathbb{E} \left [ \norm{T_{1,m}^{(i)}(\bm{X}) - \mathbb{E}[\bm{Y}]}^2\right ] \right ] = \mathbb{E} \left [ \norm{T_{1,m}(\bm{X}_i)-\mathbb{E}[\bm{Y}]}^2\right ].
    \end{equation*}
    Now using the above we have
    \begin{align*}
        \frac{1}{m}\sum_{i=1}^m \mathbb{E}_{\bm{X}} \left [ \mathbb{E} \left [ \norm{T_{1,m}^{(i)}(\bm{X}) - \mathbb{E}[\bm{Y}]}^2\right ] \right ] &= \frac{1}{m}\sum_{i=1}^m  \mathbb{E} \left [ \norm{T_{1,m}(\bm{X}_i)-\mathbb{E}[\bm{Y}]}^2\right ] \\
        &=   \mathbb{E} \left [ \frac{1}{m}\sum_{i=1}^m \norm{T_{1,m}(\bm{X}_i) -\mathbb{E}[\bm{Y}]}^2\right ].
    \end{align*}
    From here we will use equation \eqref{eq:Tn_formula} for $T_{1,m}(\bm{X}_i)$ and \eqref{eq:pn_sum} to write
    \begin{equation*}
        T_{1,m}(\bm{X}_i) - \mathbb{E}[\bm{Y}] = \sum_{j=1}^m \frac{1}{m}p_{1,m}(\bm{X}_i,\bm{Y}_j) (\bm{Y_j} - \mathbb{E}[\bm{Y}])
    \end{equation*}
    where the notation $p_{1,m}$ is intended to mean the relative density $p_m$ on batch 1. 

    From this formula we see that if we let $\mathbf{Y} \in \mathbb{R}^{m \times d}$ be the matrix with $j$'th row $\bm{Y}_j - \mathbb{E}[\bm{Y}]$ and let $\mathbf{P} \in \mathbb{R}^{m \times m}$ be the matrix with $\mathbf{P}_{i,j} = \frac{1}{m}p_{1,m}(\bm{X}_i,\bm{Y}_j)$ then  
    \begin{equation*}
        T_{1,m}(\bm{X}_i)- \mathbb{E}[\bm{Y}] = [\mathbf{P}\mathbf{Y}]_i
    \end{equation*}
    the $i$'th column of the matrix multiplication. Therefore
    \begin{equation*}
        \sum_{i=1}^m \norm{T_{1,m}(\bm{X}_i) -\mathbb{E}[\bm{Y}]}^2 = \sum_{i=1}^n \norm{[\mathbf{P}\mathbf{Y}]_i}^2 = \norm{\mathbf{P}\mathbf{Y}}_F^2
    \end{equation*}
    where $\norm{\cdot}_F$ is the Frobenius norm of a matrix. In addition, $\mathbf{P}$ is a doubly stochastic matrix and as such $\mathbf{P} \in B_m$ the Birkhoff polytope \cite{birkhoff1946tres}.
    Since the function
    \begin{equation*}
        \mathbf{A} \mapsto \norm{\mathbf{A}\mathbf{Y}}_F^2
    \end{equation*}
    is convex we have
    \begin{equation*}
        \norm{\mathbf{P}\mathbf{Y}}_F^2 \leq \max_{\mathbf{B} \in B_n} \norm{\mathbf{B}\mathbf{Y}}_F^2 = \norm{\mathbf{Y}}_F^2
    \end{equation*}
    which follows from Bauer's Maximum Principle \cite{bauer1958minimalstellen}, which states that the maximum of a convex function on a convex set occurs at the extreme points, and the fact that the extreme points of $B_n$ are permutation matrices \cite{birkhoff1946tres}. 
    Combining these facts and applying them above we have
    \begin{align*}
        \mathbb{E} \left [ \frac{1}{m}\sum_{i=1}^m \norm{T_{1,n}(\bm{X}_i)- \mathbb{E}[\bm{Y}]}^2\right ] &\leq \mathbb{E} \left [ \frac{1}{m} \norm{\mathbf{Y}}_F^2 \right ] \\
        &= \frac{1}{m} \sum_{i=1}^m \mathbb{E}[ \norm{\bm{Y}_i - \mathbb{E}[\bm{Y}]}^2 ] \\
        &= \mathbb{E}[ \norm{\bm{Y}- \mathbb{E}[\bm{Y}]}^2].
    \end{align*}
    Combining this with chain of the various bounds above completes the proof. 
    
    Under the assumption that $\norm{\bm{Y}}_{\psi_2} < \infty$ we have the bound
    \begin{align*}
        \mathbb{E}[\norm{\bm{Y}-\mathbb{E}[\bm{Y}]}^2] &\leq \mathbb{E}[\norm{\bm{Y}}^2] \\
        &= 2d\norm{\bm{Y}}_{\psi_2}^2\mathbb{E}\left [\log \left ( \exp \frac{\norm{\bm{Y}}^2}{2d\norm{\bm{Y}}_{\psi_2}^2} \right ) \right] \\
        &\leq 2d\norm{\bm{Y}}_{\psi_2}^2\log \left (\mathbb{E} \left [ \frac{\norm{\bm{Y}}^2}{2d\norm{\bm{Y}}_{\psi_2}^2} \right ] \right ) \leq 2\log(2)d\norm{\bm{Y}}_{\psi_2}^2.
    \end{align*}
    The first inequality follows from the fact that the expected value minimizes the variance. This proves the additional result.
\end{proof}

\subsection{Proof of Proposition \ref{prop:pi_tilde_properties}} \label{sec:pi_tilde_properties}

\begin{proof}
    Let $\phi:\mathbb{R}^d \times \mathbb{R}^d \rightarrow \mathbb{R}$ be a continuous and bounded function. By definition of $\tilde{\pi}$ we have 
    \begin{align*}
        \mathbb{E}_{(\bm{X},\bm{Y}) \sim \tilde{\pi}}[\phi(\bm{X},\bm{Y})] &= \mathbb{E}_{\cX,\cY}[\mathbb{E}_{(\bm{X},\bm{Y}) \sim \pi_m}[\phi(\bm{X},\bm{Y})]] \\
        &= \mathbb{E}_{\cX,\cY}\left [ \frac{1}{m^2}\sum_{i,j=1}^m \phi(\bm{X}_i,\bm{Y}_j) p_m(\bm{X}_i, \bm{Y}_j)\right ] \\
        &= \frac{1}{m^2}\sum_{i,j=1}^m \mathbb{E}_{\cX,\cY}\left [ \phi(\bm{X}_i,\bm{Y}_j) p_m(\bm{X}_i, \bm{Y}_j) \right ].
    \end{align*}
    Now because of the independence of every variable we can replace $\bm{X}_i$ and $\bm{Y}_j$ and condition on them
    \begin{align*}
        \mathbb{E}_{\cX,\cY}\left [ \phi(\bm{X}_i,\bm{Y}_j) p_m(\bm{X}_i, \bm{Y}_j) \right ] &= \mathbb{E}_{\bm{X},\bm{Y}} \left [ \mathbb{E}_{\cX,\cY}\left [ \phi(\bm{X}_i,\bm{Y}_j) p_m(\bm{X}_i, \bm{Y}_j) \ \big | \ \bm{X}_i = \bm{X}, \bm{Y}_j = \bm{Y} \right ] \right ] \\
        &= \mathbb{E}_{\bm{X},\bm{Y}} \left [ \mathbb{E}_{\cX,\cY}\left [ \phi(\bm{X}_1,\bm{Y}_1) p_m(\bm{X}_1, \bm{Y}_1) \ \big | \ \bm{X}_1 = \bm{X}, \bm{Y}_1 = \bm{Y} \right ] \right ].
    \end{align*}
    \eqref{eq:intuitive_meaning} follows by summing these expressions.

    To show that the first marginal is $\mu$ we have to check that for every test function $\phi' : \mathbb{R}^d \rightarrow \mathbb{R}$ we have
    \begin{equation*}
        \mathbb{E}_{\bm{X}}[\phi'(\bm{X})] = \mathbb{E}_{(\bm{X},\bm{Y}) \sim \tilde{\pi}}[\phi'(\bm{X})].
    \end{equation*}
    By the calculations above, we have for the test function $\phi(\bm{X},\bm{Y}) = \phi'(\bm{X})$
    \begin{align*}
        \mathbb{E}_{(\bm{X},\bm{Y}) \sim \tilde{\pi}}[\phi(\bm{X},\bm{Y})] &= \frac{1}{m^2}\sum_{i,j=1}^m\mathbb{E}_{\cX,\cY}\left [ \phi(\bm{X}_i,\bm{Y}_j) p_m(\bm{X}_i,\bm{Y}_j) \right ] \\
        &= \frac{1}{m}\sum_{i=1}^m\mathbb{E}_{\cX,\cY}\left [ \phi'(\bm{X}_i) \sum_{j=1}^m \frac{1}{m} p_m(\bm{X}_i,\bm{Y}_j) \right ] \\
        &= \frac{1}{m}\sum_{i=1}^m\mathbb{E}_{\cX,\cY}\left [ \phi'(\bm{X}_i) \right ] = \mathbb{E}_{\bm{X}}[\phi'(\bm{X})]
    \end{align*}
    where we have leveraged \eqref{eq:pn_sum} to remove the sum in the jump to the last line. The proof that the second marginal is $\nu$ is analogous. 

    The result for the conditional distribution follows by comparing the joint and marginal.
\end{proof}

\subsection{Proof of Corollary \ref{cor:pi_tilde_cond_exp}} \label{sec:pi_tilde_cond_exp}

\begin{proof}
    From Proposition \ref{prop:pi_tilde_properties} we have
    \begin{align*}
        \mathbb{E}_{\bm{Y} \sim \tilde{\pi}(\cdot|\bm{x})}[\bm{Y}] 
        &= \mathbb{E}_{\bm{Y}} \left [ \bm{Y} \frac{1}{m^2}\sum_{i,j=1}^m \mathbb{E}_{\cX,\cY}\left [ p_m(\bm{X}_i,\bm{Y}_j) \ \big | \ \bm{X}_i = \bm{x}, \bm{Y}_j = \bm{Y} \right ] \right ] \\
        &= \mathbb{E}_{\bm{Y}} \left [ \frac{1}{m^2}\sum_{i,j=1}^m \mathbb{E}_{\cX,\cY}\left [ \bm{Y}_j p_m(\bm{X}_i,\bm{Y}_j) \ \big | \ \bm{X}_i = \bm{x}, \bm{Y}_j = \bm{Y} \right ] \right ] \\
        &= \frac{1}{m^2}\sum_{i,j=1}^m \mathbb{E}_{\bm{Y}} \left [ \mathbb{E}_{\cX,\cY}\left [ \bm{Y}_j p_m(\bm{X}_i,\bm{Y}_j) \ \big | \ \bm{X}_i = \bm{x}, \bm{Y}_j = \bm{Y} \right ] \right ] 
    \end{align*}
    As mentioned above, by independence and identical distribution of $\bm{Y}$ and $\bm{Y}_j$ we have
    \begin{equation*}
        \mathbb{E}_{\bm{Y}} \left [ \mathbb{E}_{\cX,\cY}\left [ \bm{Y}_j p_m (\bm{X}_i,\bm{Y}_j) \ \big | \ \bm{X}_i = \bm{x}, \bm{Y}_j = \bm{Y} \right ] \right ] =  \mathbb{E}_{\cX,\cY}\left [ \bm{Y}_j p_m (\bm{X}_i,\bm{Y}_j) \ \big | \ \bm{X}_i = \bm{x} \right ].
    \end{equation*}
    Applying this to every term in the summation and continuing the calculation above gives
    \begin{align*}
        \mathbb{E}_{\bm{Y} \sim \tilde{\pi}(\cdot|\bm{x})}[\bm{Y}]  &= \frac{1}{m^2}\sum_{i,j=1}^m  \mathbb{E}_{\cX,\cY}\left [ \bm{Y}_j p_m(\bm{X}_i,\bm{Y}_j) \ \big | \ \bm{X}_i = \bm{x} \right ]  \\
        &= \frac{1}{m}\sum_{i=1}^m  \mathbb{E}_{\cX,\cY}\left [ \sum_{j=1}^m \bm{Y}_j \frac{1}{m} p_m(\bm{X}_i,\bm{Y}_j) \ \big | \ \bm{X}_i = \bm{x} \right ]  \\
        &= \frac{1}{m} \sum_{i=1}^m  \mathbb{E}_{\cX,\cY} \left [T_m(\bm{X_i}) \ \big | \ \bm{X}_i = \bm{x} \right ] \\
        &= \frac{1}{m} \sum_{i=1}^m  \mathbb{E}_{\cX,\cY} \left [T_{1,m}^{(i)}(\bm{x}) \right ].
    \end{align*}
    By exchangeability it holds for $i=1,...,m$ that
    \begin{equation*}
         \mathbb{E}_{\cX,\cY} \left [T_{1,m}^{(i)}(\bm{x}) \right ] =  \mathbb{E}_{\cX,\cY} \left [T_{1,m}^{(1)}(\bm{x}) \right ].
    \end{equation*}
    Plugging this in above gives
    \begin{align*}
        \mathbb{E}_{\bm{Y} \sim \tilde{\pi}(\cdot|\bm{x})}[\bm{Y}] = \frac{1}{m}\sum_{i=1}^m \mathbb{E} \left [ T^{(i)}_{1,m}(\bm{x}) \right ] = \mathbb{E}_{\cX,\cY} \left [T^{(1)}_{1,m}(\bm{x}) \right ].
    \end{align*}
    To conclude we have by linearity and identical distributions of $T_{1,m}^{(1)}(\bm{x}),...,T_{k,m}^{(1)}(\bm{x})$ that 
    \begin{equation*}
        \mathbb{E}[\hat{T}_n(\bm{x})] =  \mathbb{E}\left [ \frac{1}{k} \sum_{\ell=1}^k T^{(1)}_{\ell,m}(\bm{x}) \right ] = \mathbb{E}[T^{(1)}_{1,m}(\bm{x})] = \mathbb{E}_{\bm{Y} \sim \tilde{\pi}(\cdot|\bm{x})}[\bm{Y}].
    \end{equation*}
\end{proof}

\subsection{Proof of Proposition \ref{prop:kl_bound}} \label{sec:kl_bound}

\begin{proof}
    For the first equality note that for any two probability measures $\alpha,\beta$ over$\mathbb{R}^d \times \mathbb{R}^d$ with first marginals (over $\mathbb{R}^{d}$) $\alpha_1, \beta_1$ a direct calculation shows
    \begin{equation*}
        \alpha(\mathsf{KL}(\alpha(\cdot|\bm{X}) \ || \ \beta(\cdot|\bm{X}))) = \mathsf{KL}(\alpha\  || \ \beta)) - \mathsf{KL}(\alpha_1 \ || \ \beta_1)
    \end{equation*}
    so if $\alpha_1 = \beta_1$ then $\mathsf{KL}(\alpha_1 \ || \ \beta_1) = 0$ and 
    \begin{equation*}
         \alpha(\mathsf{KL}(\alpha(\cdot|\bm{X}) \ || \ \beta(\cdot|\bm{X}))) = \mathsf{KL}(\alpha\  || \ \beta)).
    \end{equation*}
    Applying this with $\alpha = \tilde{\pi}$ and $\beta = \pi_\varepsilon$ the first equality will follow if we can show that these measures have the same first marginal, in this case $\mu$. For $\pi_\varepsilon$ this is clear from the coupling constraint, and for $\tilde{\pi}$ the result follows from Proposition \ref{prop:pi_tilde_properties}.

    We now start working on the expression for the KL-Divergence. Using Proposition \ref{prop:pi_tilde_properties} and equation \eqref{eq:opt_relation} we can express the KL-Divergence, then apply Jensen's inequality twice:
    \begin{align*}
        & \mathsf{KL}(\tilde{\pi}(\bm{\cdot}|\bm{x})  \ || \ \pi_\varepsilon(\bm{\cdot}|\bm{x})) \notag \\
        & = \mathbb{E}_{\bm{Y}}\left [\left(\frac{1}{m^2} \sum_{i, j=1}^{m} \mathbb{E}_{\cX, \cY}\left [p_m(\bm{X}_i, \bm{Y}_j) \ | \ \bm{X}_i = \bm{x}, \bm{Y}_j = \bm{Y})\right ] \right)\log \frac{\frac{1}{m^2} \sum_{i, j=1}^{m} \mathbb{E}_{\cX, \cY}[p_m(\bm{X}_i, \bm{Y}_j) \ | \ \bm{X}_i = \bm{x}, \bm{Y}_j = \bm{Y})]}{p_\varepsilon(\bm{x}, \bm{Y})} \right ]\\
        & \leq \mathbb{E}_{\bm{Y}}\left [ \frac{1}{m^2} \sum_{i, j=1}^{m} \left(\mathbb{E}_{\cX, \cY}[p_m(\bm{X}_i, \bm{Y}_j) \ | \ \bm{X}_i = \bm{x}, \bm{Y}_j = \bm{Y})] \log \frac{ \mathbb{E}_{\cX, \cY}[p_m(\bm{X}_i, \bm{Y}_j) \ | \ \bm{X}_i = \bm{x}, \bm{Y}_j = \bm{Y})]}{p_\varepsilon(\bm{x}, \bm{Y})}\right) \right ]\\
        & \leq \mathbb{E}_{\bm{Y}} \left [ \frac{1}{m^2} \sum_{i,j=1}^{m} \mathbb{E}_{\cX, \cY}\left[ p_m(\bm{X}_i, \bm{Y}_j) \log \frac{p_m(\bm{X}_i, \bm{Y}_j)}{p_\varepsilon(\bm{x}, \bm{Y})} \ \bigg | \ \bm{X}_i = \bm{x}, \bm{Y}_j = \bm{Y} \right] \right ] \\
        & = \frac{1}{m^2} \sum_{i,j=1}^{m} \mathbb{E}_{\bm{Y}} \left [ \mathbb{E}_{\cX, \cY} \left [ p_m(\bm{X}_i, \bm{Y}_j) \log \frac{p_m(\bm{X}_i, \bm{Y}_j)}{p_\varepsilon(\bm{x}, \bm{Y}_j)} \ \bigg | \ \bm{X}_i = \bm{x},\bm{Y}_j = \bm{Y} \right] \right ].
    \end{align*}
    
    Now as in the proof of Proposition \ref{prop:pi_tilde_properties}, we have
    \begin{equation*}
        \mathbb{E}_{\bm{Y}} \left [ \mathbb{E}_{\cX, \cY} \left [ p_m(\bm{X}_i, \bm{Y}_j) \log \frac{p_m(\bm{X}_i, \bm{Y}_j)}{p_\varepsilon(\bm{x}, \bm{Y}_j)} \ \bigg | \ \bm{X}_i = \bm{x},\bm{Y}_j = \bm{Y} \right] \right ] =  \mathbb{E}_{\cX, \cY} \left [ p_m(\bm{X}_i, \bm{Y}_j) \log \frac{p_m(\bm{X}_i, \bm{Y}_j)}{p_\varepsilon(\bm{x}, \bm{Y}_j)} \ \bigg | \ \bm{X}_i = \bm{x}\right] 
    \end{equation*}
    Applying this identity and using linearity we have
    \begin{align*}
        \mathsf{KL}(\tilde{\pi}(\cdot | \bm{x}) \ || \  \pi_\varepsilon(\cdot  |  \bm{x})) \leq \frac{1}{m^2} \sum_{i,j=1}^m  \mathbb{E}_{\cX, \cY} \left [ p_m(\bm{X}_i, \bm{Y}_j) \log \frac{p_m(\bm{X}_i, \bm{Y}_j)}{p_\varepsilon(\bm{x}, \bm{Y}_j)} \ \bigg | \ \bm{X}_i = \bm{x}\right].
    \end{align*}
    This holds for every fixed $\bm{x}$. In particular this implies:
    \begin{align*}
        \mathbb{E}_{\bm{X}}\left [\mathsf{KL}(\tilde{\pi}(\bm{\cdot}|\bm{X}) || \pi_\varepsilon(\cdot|\bm{X}))  \right ] &\leq \mathbb{E}_{\bm{X}}\left [ \frac{1}{m^2} \sum_{i,j=1}^m  \mathbb{E}_{\cX, \cY} \left [ p_m(\bm{X}_i, \bm{Y}_j) \log \frac{p_m(\bm{X}_i, \bm{Y}_j)}{p_\varepsilon(\bm{X}, \bm{Y}_j)} \ \bigg | \ \bm{X}_i = \bm{X} \right] \right ] \\
        &= \frac{1}{m^2}\sum_{i,j=1}^m \mathbb{E}_{\bm{X}}\left [  \mathbb{E}_{\cX, \cY} \left [ p_m(\bm{X}_i, \bm{Y}_j) \log \frac{p_m(\bm{X}_i, \bm{Y}_j)}{p_\varepsilon(\bm{X}_i, \bm{Y}_j)} \ \bigg | \ \bm{X}_i = \bm{X} \right] \right ] \\
        &= \frac{1}{m^2}\sum_{i,j=1}^m \mathbb{E}_{\cX, \cY} \left [ p_m(\bm{X}_i, \bm{Y}_j) \log \frac{p_m(\bm{X}_i, \bm{Y}_j)}{p_\varepsilon(\bm{X}_i, \bm{Y}_j)} \right] \\
        &= \mathbb{E}_{\cX, \cY} \left [ \frac{1}{m^2}\sum_{i,j=1}^m p_m(\bm{X}_i, \bm{Y}_j) \log \frac{p_m(\bm{X}_i, \bm{Y}_j)}{p_\varepsilon(\bm{X}_i, \bm{Y}_j)} \right]
    \end{align*}
    where on the third line we have used the same sample swap argument except this time applied to $\bm{X}$.

    Now we turn our attention to the expression we are taking expectations of.
    To start we have by definition of $p$ and $p_m$
    \begin{align*}
        \log \frac{p_m(\bm{X}_i, \bm{Y}_j)}{p_\varepsilon(\bm{X}_i, \bm{Y}_j)} &= \frac{1}{\varepsilon} \left (f_m(\bm{X}_i) + g_m(\bm{Y}_j) - \frac{1}{2}\norm{\bm{X}_i - \bm{Y}_j}^2 \right ) - \frac{1}{\varepsilon} \left ( f_\varepsilon(\bm{X}_i) + g_\varepsilon(\bm{Y}_j) - \frac{1}{2}\norm{\bm{X}_i - \bm{Y}_j}^2\right ) \\
        &= \frac{1}{\varepsilon} \left ( f_m(\bm{X}_i) + g_m(\bm{Y}_j) - f_\varepsilon(\bm{X}_i) - g_\varepsilon(\bm{Y}_j)  \right ).
    \end{align*}
    We also have by \eqref{eq:pn_sum}
    \begin{align*}
        &\sum_{i,j=1}^m \frac{1}{m^2} p_m(\bm{X}_i, \bm{Y}_j) (f_m(\bm{X}_i) + g_m(\bm{Y}_j)) \\
        &= \sum_{i=1}^m \frac{1}{m}f_m(\bm{X}_i)\left ( \sum_{j=1}^m \frac{1}{m} p_m(\bm{X}_i, \bm{Y}_j) \right ) + \sum_{j=1}^m \frac{1}{m}g_m(\bm{Y}_j)\left ( \sum_{i=1}^m \frac{1}{m} p_m(\bm{X}_i, \bm{Y}_j) \right ) \\
        &=  \sum_{i=1}^m \frac{1}{m}f_m(\bm{X}_i) + \sum_{j=1}^m \frac{1}{m}g_m(\bm{Y}_j) \\
        &= \mu_{\cX}(f_m) + \nu_{\cY}(g_m).
    \end{align*}
    In addition 
    \begin{align*}
        &-\varepsilon(\mu_{\cX} \otimes \nu_{\cY}) \left ( \expe \left ( f_m + g_m - \frac{1}{2}\norm{\bm{x} - \bm{y}}^2\right )\right ) + \varepsilon \\
        &= \varepsilon -\varepsilon \frac{1}{m^2} \sum_{i,j=1}^m p_m{\bm{X}_i,\bm{Y}_j} \\
        &= \varepsilon - \varepsilon \frac{1}{m} \sum_{i=1}^m \frac{1}{m} \sum_{j=1}^m p_m(\bm{X}_i,\bm{Y}_j) = \varepsilon - \varepsilon \frac{1}{m} \sum_{i=1}^m 1 = 0.
    \end{align*}
    Combining the previous two calculations we have
    \begin{equation*}
        \sum_{i,j=1}^m \frac{1}{m^2} p_m(\bm{X}_i, \bm{Y}_j) (f_m(\bm{X}_i) + g_m(\bm{Y}_j)) = S_\varepsilon(\cX,\cY).
    \end{equation*}
    
    In addition
    \begin{align*}
        &\mathbb{E}_{\cX,\cY} \left [ \frac{1}{m^2} \sum_{i,j=1}^m p_m(\bm{X}_i, \bm{Y}_j)(f_\varepsilon(\bm{X}_i) + g_\varepsilon(\bm{Y}_j))  \right ] \\
        &= \mathbb{E}_{\cX,\cY} \left [ \sum_{i=1}^m f_\varepsilon(\bm{X}_i) \sum_{j=1}^m \frac{1}{m^2} p_m(\bm{X}_i, \bm{Y}_j)  \right ]  + \mathbb{E}_{\cX,\cY} \left [ \sum_{j=1}^m g_\varepsilon(\bm{Y}_j) \sum_{i=1}^m \frac{1}{m^2} p_m(\bm{X}_i, \bm{Y}_j)  \right ] \\
        &=  \mathbb{E}_{\cX,\cY} \left [ \sum_{i=1}^m f_\varepsilon(\bm{X}_i) \frac{1}{m}  \right ]  + \mathbb{E}_{\cX,\cY} \left [ \sum_{j=1}^m g_\varepsilon(\bm{Y}_j) \frac{1}{m} \right ] \\
        &= \mathbb{E}_{\bm{X}} [f_\varepsilon(\bm{X})] + \mathbb{E}_{\bm{Y}} [g_\varepsilon(\bm{Y})] = S_\varepsilon(\mu,\nu)
    \end{align*}
    where we have made use of the \eqref{eq:pn_sum} to deal with the summations.  

    Collecting inequalities above we have shown
    \begin{equation*}
        \mathbb{E}_{\bm{X}}\left [\mathsf{KL}(\tilde{\pi}(\bm{\cdot}|\bm{X}) || \pi_\varepsilon(\cdot|\bm{X}))  \right ] \leq \frac{1}{\varepsilon} \left ( \mathbb{E}S_\varepsilon{(\cX,\cY)} - S_\varepsilon(\mu,\nu) \right ) \leq \frac{1}{\varepsilon} \mathbb{E} \left [\left | S_\varepsilon{(\cX,\cY)} - S_\varepsilon(\mu,\nu) \right | \right ].
    \end{equation*}
 \end{proof}

\section{Proofs From Section \ref{sec:bounded_slc}}

\subsection{Proof of Theorem \ref{thm:ultimate_bounded}}

\begin{proof}
    We start by applying Lemma \ref{lem:triangle}:
    \begin{equation*}
        \mathbb{E}_{\cX,\cY} [\|\hat{T}_n - T_\varepsilon \|^2_{L^2(\mu)}] \leq 2\mathbb{E}_{\cX,\cY} [\|\hat{T}_n - \mathbb{E}[\hat{T}_n]\|_{L^2(\mu)}^2]  
         +2 \|\mathbb{E}[\hat{T}_n] - T_\varepsilon\|_{L^2(\mu)}^2.
    \end{equation*}
    The first term is handled by Lemma \ref{lem:easy_term} and the variational formula for variance:
    \begin{equation*}
        \mathbb{E}_{\cX,\cY} [\|\hat{T}_n - \mathbb{E}[\hat{T}_n]\|_{L^2(\mu)}^2] \leq \mathbb{E}[\norm{\bm{Y} - \mathbb{E}\bm{Y}}^2]/k \leq \mathbb{E}[\norm{\bm{Y}}^2]/k = R^2/k.
    \end{equation*}
    The latter term is handled by the following chain
    \begin{align*}
        \|\mathbb{E}[\hat{T}_n] - T_\varepsilon\|_{L^2(\mu)}^2 &\leq \mathbb{E}_{\bm{X}} W_1^2(\tilde{\pi}(\cdot|\bm{X}), \pi_{\varepsilon}(\cdot|\bm{X})) & (\text{Eq.  \ref{eq:expectation_vs_w1}}) \\
        &\leq \mathbb{E}_{\bm{X}}8R\cdot \mathsf{KL}(\tilde{\pi}(\cdot|\bm{X}) \ || \ \pi_{\varepsilon}(\cdot|\bm{X})) & (\text{Thm. \ref{thm:bounded_T1}}) \\
        &\leq \frac{1}{\varepsilon} \mathbb{E}_{\cX,\cY} |S_\varepsilon(\mu,\nu) - S_\varepsilon(\cX,\cY)| & (\text{Prop. \ref{prop:kl_bound}}) \\
        &\leq \left ( 1 + \frac{\sigma_0^{\lceil 5d/2 \rceil + 6}}{\varepsilon^{\lceil 5d/4\rceil + 3}}\right ) \cdot \frac{1}{\sqrt{m}} & (\text{Thm. \ref{thm:mena_weed}})
    \end{align*}
\end{proof}

\subsection{Proof of Theorem \ref{thm:strongly_lc_lap}} \label{sec:strongly_lc_lap}

The proof of Theorem \ref{thm:strongly_lc_lap} follows immediately from the following two definitions and theorems. This material is extracted from Chapter 5 in \cite{ledoux2001concentration}.

\begin{definition} Given a probability measure $\mu$ on some measurable space $(\Omega, \Sigma)$, for every non-negative measurable function $f$ on $(\Omega,\Sigma)$ define its \textit{entropy} as
\begin{equation*}
    \text{Ent}_\mu(f) = \int f\log f d\mu - \left ( \int f d\mu\right ) \log \left ( \int f d\mu \right )
\end{equation*}
if $\int f\log(1+f)d\mu < \infty$ and $+\infty$ if not.
\end{definition}

\begin{definition}
    A function $f:\mathbb{R}^d \rightarrow \mathbb{R}$ is said to be locally Lipschitz if for every $x \in \mathbb{R}^d$ there exists a $\delta > 0$ such that
    \begin{equation*}
        \sup_{y \in B(0,\delta), y \neq x} \frac{|f(x) - f(y)|}{\norm{x - y}} < \infty.
    \end{equation*}
    For a locally Lipschitz function we will use the notation
    \begin{equation*}
        |\nabla f(x)| = \limsup_{y \rightarrow x} \frac{|f(x) - f(y)|}{\norm{x-y}}.
    \end{equation*}
\end{definition}

\begin{definition} A probability measure $\mu$ on the Borel sets of $\mathbb{R}^n$ is said to satisfy a \textit{logarithmic Sobolev inequality} if for some constant $C > 0$ and all locally Lipschitz functions $f$ on $\mathbb{R}^d$,
\begin{equation} \label{eq:LSI}
    \text{Ent}_\mu(f^2) \leq 2C \int |\nabla f|^2 d\mu.
\end{equation}
\end{definition}

\begin{theorem}\label{thm:log_concave_implies_lsi} (\cite{ledoux2001concentration}Theorem 5.2) Let $d\mu = e^{-U}dx$ where, for some $c > 0$, $\nabla^2 U(x) \succeq c\text{Id}$ uniformly in $x \in \mathbb{R}^d.$ Then for all locally Lipschitz functions $f$ on $\mathbb{R}^d$,
\begin{equation*}
    \text{Ent}_\mu(f^2) \leq \frac{2}{c}\int | \nabla f |^2 d\mu. 
\end{equation*}
\end{theorem}

\begin{theorem} \label{thm:LSI_controls_lap} Let $\mu$ be a probability measure on $\mathbb{R}^d$ and satisfy $(\ref{eq:LSI})$ for some constant $C > 0$. Then 
\begin{equation*}
    E_{\mu}(\lambda) \leq e^{C\lambda^2/2}, \hspace{0.5cm} \lambda \geq 0.
\end{equation*}
\end{theorem}

We can now proceed to the proof of Theorem \ref{thm:strongly_lc_lap} which combines the results above.
\begin{proof}[Proof of Theorem \ref{thm:strongly_lc_lap}] 
    Since $\nu$ has density $d\nu = e^{-W}dx$ and $\nabla^2U(y) \succeq c\text{Id}$ for $c > 0$ we have by Theorem \ref{thm:log_concave_implies_lsi} that $\nu$ satisfies
    \begin{equation*}
        \text{Ent}_\nu(f^2) \leq \frac{2}{c}\int | \nabla f |^2 d\nu
    \end{equation*}
    for all locally Lipschitz functions $f$ on $\mathbb{R}^d$. Therefore $\nu$ satisfies a logarithmic Sobolev inequality with constant $\frac{1}{c}$. Now by Theorem \ref{thm:LSI_controls_lap} we have that the Laplace functional of $\nu$ is controlled by
    \begin{equation*}
        E_{(\nu)}(\lambda) \leq e^{\frac{1}{c}\lambda^2/2}, \hspace{0.5cm} \lambda \geq 0
    \end{equation*}
    which proves the result.
\end{proof}

\subsection{Proof of Lemma \ref{lem:conditionals_strongly_lc}} \label{sec:conditionals_strongly_lc}

\begin{proof}
    Note that 
    \begin{align*}
        d\pi_{\varepsilon(\bm{y}|\bm{x})} &\propto \expe \left ( f_\varepsilon(\bm{x}) + g_\varepsilon(\bm{y}) - \frac{1}{2}\norm{\bm{x} - \bm{y}}^2 \right ) e^{-W(\bm{y})} d\bm{y} \\
        &\propto \exp\left ( -\frac{1}{\varepsilon}\left (\frac{1}{2}\norm{\bm{y}}^2 - g_\varepsilon(\bm{y}) - \langle \bm{y}, \bm{x} \rangle \right ) -  W(\bm{y}) \right )  d\bm{y}
    \end{align*}
    Therefore it is enough to show that $\bm{y} \mapsto \frac{1}{2}\norm{\bm{y}}^2 - g_\varepsilon(\bm{y}) - \langle \bm{y}, \bm{x} \rangle$ is convex, which is equivalent to showing that $\bm{y} \mapsto \frac{1}{2}\norm{\bm{y}}^2 - g_\varepsilon(\bm{y})$ is convex. This is known in some cases and is stated  in \cite{chewi2022entropic} Lemma 1. We prove this with minimal restrictions on the measure $\mu$.

    Recall that \begin{align}
        \bm{g}_\varepsilon(\bm{y}) &= -\varepsilon\log \mu \left ( \expe\left ( f_\varepsilon(\bm{x}) - \frac{1}{2} \norm{\bm{y} - \bm{x}}^2 \right ) \right ) \notag \\
        &= -\varepsilon\log \mu \left ( \expe\left ( f_\varepsilon(\bm{x}) - \frac{1}{2} \norm{\bm{x}}^2 + \langle \bm{y},\bm{x} \rangle \right ) \right ) + \frac{1}{2}\norm{\bm{y}}^2 \notag \\
        \implies \frac{1}{2}\norm{\bm{y}}^2 - g_\varepsilon(\bm{y}) &= \varepsilon\log \mu \left ( \expe\left ( f_\varepsilon(\bm{x}) - \frac{1}{2} \norm{\bm{x}}^2 + \langle \bm{y},\bm{x} \rangle \right ) \right ) \label{eq:norm_minus_g}
    \end{align}
    Now let $\bm{y}_0, \bm{y}_1 \in \mathbb{R}^d, \lambda \in (0,1)$ and set $\bm{y}_\lambda = (1-\lambda)\bm{y}_0 + \lambda\bm{y}_1$ so that by \eqref{eq:norm_minus_g} applied to $\bm{y}_\lambda$ we have
    \begin{align*}
        \frac{1}{2}\norm{\bm{y}_\lambda}^2 - g_\varepsilon(\bm{y}_\lambda) &= \varepsilon\log \mu \left ( \expe\left ( f_\varepsilon(\bm{x}) - \frac{1}{2} \norm{\bm{x}}^2 + \langle \bm{y}_\lambda,\bm{x} \rangle \right ) \right ) \\
        &= \varepsilon\log \mu \left ( \expe\left ( (1-\lambda)\left (  f_\varepsilon(\bm{x}) - \frac{1}{2} \norm{\bm{x}}^2 + \langle \bm{y}_0,\bm{x} \rangle \right ) + \lambda\left (  f_\varepsilon(\bm{x}) - \frac{1}{2} \norm{\bm{x}}^2 + \langle \bm{y}_1,\bm{x} \rangle \right ) \right ) \right ) 
    \end{align*}
    Now a version of Hölder's inequality states that for any random variables $\bm{U},\bm{V}$ and $\lambda \in (0,1)$ it holds
    \begin{equation}
        \log \mathbb{E}\left [ e^{(1-\lambda) \bm{U} + \lambda \bm{V})}\right ] \leq (1-\lambda) \log \mathbb{E}[e^{\bm{U}}] + \lambda \mathbb{E}[e^{\bm{V}}]. \label{eq:holder_special}
    \end{equation}
    Applying \eqref{eq:holder_special} to the random variables
    \begin{align*}
        \bm{U} &= \frac{1}{\varepsilon} \left ( f_\varepsilon(\bm{X}) - \frac{1}{2} \norm{\bm{X}}^2 + \langle \bm{y}_0,\bm{X} \rangle \right ) \\
        \bm{V} &= \frac{1}{\varepsilon} \left ( f_\varepsilon(\bm{X}) - \frac{1}{2} \norm{\bm{X}}^2 + \langle \bm{y}_1,\bm{X} \rangle \right )
    \end{align*}
    above gives
    \begin{align*}
        &\varepsilon\log \mu \left ( \expe\left ( (1-\lambda)\left (  f_\varepsilon(\bm{x}) - \frac{1}{2} \norm{\bm{x}}^2 + \langle \bm{y}_0,\bm{x} \rangle \right ) + \lambda\left (  f_\varepsilon(\bm{x}) - \frac{1}{2} \norm{\bm{x}}^2 + \langle \bm{y}_1,\bm{x} \rangle \right ) \right ) \right ) \\
        &\leq (1-\lambda) \varepsilon \log \mu \left ( \expe \left ( f_\varepsilon(\bm{x}) - \frac{1}{2} \norm{\bm{x}}^2 + \langle \bm{y}_0,\bm{X} \rangle \right ) \right ) + \lambda \varepsilon \log \mu \left ( \expe \left ( f_\varepsilon(\bm{x}) - \frac{1}{2} \norm{\bm{x}}^2 + \langle \bm{y}_1,\bm{X} \rangle \right ) \right ) \\
        &= (1-\lambda) \left ( \frac{1}{2}\norm{\bm{y}_0}^2 - g_\varepsilon(\bm{y}_0) \right ) + \lambda \left ( \frac{1}{2}\norm{\bm{y}_1}^2 - g_\varepsilon(\bm{y}_1) \right )
    \end{align*}
    where the last equality is two applications of \eqref{eq:norm_minus_g}. Overall this establishes
    \begin{equation*}
        \frac{1}{2}\norm{\bm{y}_\lambda}^2 - g_\varepsilon(\bm{y}_\lambda) \leq (1-\lambda) \left ( \frac{1}{2}\norm{\bm{y}_0}^2 - g_\varepsilon(\bm{y}_0) \right ) + \lambda \left ( \frac{1}{2}\norm{\bm{y}_1}^2 - g_\varepsilon(\bm{y}_1) \right )
    \end{equation*}
    which shows that $\frac{1}{2}\norm{\cdot}^2 - g_\varepsilon$ is convex.
\end{proof}

\subsection{Proof of Theorem \ref{thm:ultimate_lcc}}

\begin{proof}
    Before proceeding to the proof we collect a few facts. 

    First we have
    \begin{equation*}
        \mathbb{E}\norm{\bm{Y} - \mathbb{E}\bm{Y}}^2 = \text{Tr}\text{Cov}(\bm{Y})
    \end{equation*}
    and by the Brascamp-Lieb  inequality \cite{bobkov2000brunn} and the fact that inverses reverse the PSD ordering we have
    \begin{equation*}
        \text{Cov}(\bm{Y}) \precsim \mathbb{E}[(\nabla^2W(\bm{Y}))^{-1}] \precsim (c\text{Id})^{-1} = \frac{1}{c}\text{Id}.
    \end{equation*}
    Now using the fact that $A \precsim B \implies \text{Tr}(A) \leq \text{Tr}(B)$ we have
    \begin{equation*}
        \mathbb{E}\norm{\bm{Y} - \mathbb{E}\bm{Y}}^2 = \text{Tr}\text{Cov}(\bm{Y}) \leq \text{Tr} \frac{1}{c}\text{Id} = \frac{d}{c}.
    \end{equation*}
    This will be useful for bounding the variance term.

    By Theorem 5.2.15 in \cite{vershynin2018high} it holds that there exists a universal constant $K$ such that
    \begin{equation*}
        \mathbb{E} \left [ e^{c\bm{Y}_i^2/K} \right ] \leq 2
    \end{equation*}
    for $i=1,...,d$ and therefore
    \begin{align*}
        \mathbb{E} \left [ \exp \left ( \frac{\norm{\sqrt{c}\bm{Y}}^2}{2dK} \right )\right ] 
        &\leq \mathbb{E} \left [ \frac{1}{d}\sum_{i=1}^d \exp\left ( \frac{c\bm{Y}_i^2}{2K} \right ) \right ] \\
        &= \frac{1}{d}\sum_{i=1}^d \mathbb{E} \left [\exp\left ( \frac{c\bm{Y}_i^2}{K'} \right ) \right ] \leq 2
    \end{align*}
    and therefore $\bm{Y}$ is $(K'/\sqrt{c})$-norm-subgaussian for a universal constant $K'$.

    Now we can proceed to the main proof.
    \begin{align*}
        \mathbb{E}_{\cX,\cY}\left[ \norm{\hat{T}_n - T_\varepsilon}^2_{L^2(\mu)} \right ] 
        &\lesssim \mathbb{E}_{\cX,\cY} [\|\hat{T}_n - \mathbb{E}[\hat{T}_n]\|_{L^2(\mu)}^2]  
         + \|\mathbb{E}[\hat{T}_n] - T_\varepsilon\|_{L^2(\mu)}^2 \\
        &\lesssim \mathbb{E}[\norm{\bm{Y} - \mathbb{E}\bm{Y}}^2] +  \mathbb{E}_{\bm{X}}\|\mathbb{E}_{\bm{Y} \sim \tilde{\pi}(\cdot|\bm{X})}[\bm{Y}] - \mathbb{E}_{\bm{Y} \sim \pi_\varepsilon(\cdot|\bm{X})} [\bm{Y}]\|_{L^2(\mu)}^2 \\
        &\lesssim \frac{d}{ck} + \frac{1}{c} \mathbb{E}_{\bm{X}}[\mathsf{KL}(\tilde{\pi}(\cdot|\bm{x}) \ || \pi_\varepsilon(\cdot|\bm{x}) \ )] \\
        &\lesssim \frac{d}{ck} + \frac{1}{c}\frac{1}{\varepsilon} \mathbb{E}|S_\varepsilon(\cX,\cY) - S_\varepsilon(\mu,\nu)| \\
        &\lesssim \frac{d}{ck} + \frac{1}{c}\left ( 1 + \frac{\sigma_0^{\lceil 5d/2 \rceil + 6}}{\varepsilon^{\lceil 5d/4\rceil + 3}}\right ) \cdot \frac{1}{\sqrt{m}}.
    \end{align*}
\end{proof}

\section{Proofs From Section \ref{sec:subGaussian_measures}}

\subsection{Proof of Lemma \ref{lem:subg_lap_bound}} \label{sec:subg_lap_bound}

The proof requires several preliminary results on subGaussian random variables. 

\begin{lemma} \label{lem:change_constant}
    Let $\bm{X}$ be a scalar random variable $\sigma^2 \in \mathbb{R}$ and $K > 2$ be such that
    \begin{equation*}
        \mathbb{E} \left [ e^{\bm{X}^2/\sigma^2} \right ] \leq K.
    \end{equation*}
    Then
    \begin{equation*}
        \mathbb{E} \left [ \exp \left ( \frac{\bm{X}^2}{\sigma^2 \log(K)/ \log(2)} \right )\right ] \leq 2.
    \end{equation*}
\end{lemma}
\begin{proof}
    Using basic logarithm properties we have
    \begin{equation*}
        \frac{\log(2)}{\log(K)} = \log_K(2), \hspace{0.5cm} K^{\log_K(2)} = 2, \hspace{0.5cm} 0 = \log_K(1) < \log_K(2) < \log_K(K) = 1.
    \end{equation*}
    Importantly the last inequality implies that the function $t \mapsto t^{\log_K(2)}$ is concave.

    From this we have
    \begin{align*}
        \mathbb{E} \left [ \exp \left ( \frac{\bm{X}^2}{\sigma^2 \log(K)/ \log(2)} \right )\right ] 
        &= \mathbb{E} \left [ \exp \left ( \frac{\bm{X}^2}{\sigma^2} \right )^{\log_K(2)} \right ]
        \\
        &\leq \mathbb{E} \left [ \exp \left ( \frac{\bm{X}^2}{\sigma^2} \right ) \right ]^{\log_K(2)} & (\text{Jensen}) \\
        &\leq  K^{\log_K(2)} = 2. & (\text{By Assumption})
    \end{align*}
\end{proof}

\begin{lemma} \label{lem:sq_mz_to_not_sq}
    Let $\bm{X}$ be a scalar random variable such that $\mathbb{E}[\bm{X}] = 0$ and $\mathbb{E}[e^{\bm{X}^2/\sigma^2}] \leq 2$. Then for every $\lambda \in \mathbb{R}$ it holds
    \begin{equation*}
        \mathbb{E}[\exp(\lambda \bm{X})] \leq \exp\left ( \sigma^2 \lambda^2 \right ).
    \end{equation*}
\end{lemma}
\begin{proof}
    \begin{align*}
        \mathbb{E}[\exp(\lambda \bm{X})] &= \mathbb{E} \left [ \sum_{k=0}^\infty \frac{(\lambda \bm{X})^k}{k!} \right ] \\
        &= 1 + \mathbb{E} \left [ \sum_{k=2}^\infty \frac{(\lambda \bm{X})^k}{k!} \right ] & (\mathbb{E}\bm{X} = 0) \\
        &\leq 1 + \frac{\lambda^2}{2}\mathbb{E}\left [ \bm{X}^2 \sum_{k=0}^\infty \frac{|\lambda \bm{X}|^k}{k!}\right ] & (\frac{1}{k!} \leq \frac{1}{2}\frac{1}{(k-2)!} \text{ for } k \geq 2) \\
        &= 1 + \frac{\lambda^2}{2}\mathbb{E} \left [ \bm{X}^2\exp(|\lambda \bm{X}|)\right ] \\
        &\leq 1 + \frac{\lambda^2}{2} \mathbb{E}\left [\bm{X}^2\exp \left ( \frac{\bm{X}^2}{2\sigma^2} + \frac{\sigma^2\lambda^2}{2} \right ) \right ] & (\text{Young's Inequality}) \\
        &= 1 + \frac{\sigma^2\lambda^2}{2}e^{\sigma^2\lambda^2/2}\mathbb{E}\left [ \frac{\bm{X}^2}{\sigma^2}\exp\left( \frac{\bm{X}^2}{2\sigma^2} \right ) \right ] \\
        &\leq 1 + \frac{\sigma^2\lambda^2}{2}e^{\sigma^2\lambda^2/2}\mathbb{E}\left [ \exp\left( \frac{\bm{X}^2}{\sigma^2} \right ) \right ] & (t \leq e^{t/2}) \\
        &\leq 1 + \sigma^2\lambda^2e^{\sigma^2\lambda^2/2} & (\text{By Assumption}) \\
        &\leq e^{\lambda^2\sigma^2} & (1+xe^{x/2} \leq e^{x} \ \forall \ x \geq 0)
    \end{align*} 
\end{proof}

\begin{lemma} \label{lem:f_sq_bound}
    Let $\norm{\bm{X}}_{\psi_2}^2 < \infty$. For every mean-zero 1-Lipschitz function $f$ it holds
    \begin{equation*}
        \mathbb{E}[e^{f(\bm{X})^2/8d\norm{\bm{X}}_{\psi_2}^2}] \leq 2.
    \end{equation*}
\end{lemma}
\begin{proof}
    Let $\sigma_0$ be a constant to be determined later.  Let $\bm{X}'$ be an iid copy of $\bm{X}$. We compute:
    \begin{align*}
        \mathbb{E}\left [\exp \left( \frac{f(\bm{X})^2}{\sigma_0^2} \right ) \right ] &= \mathbb{E}\left [\exp \left( \frac{(f(\bm{X}) - \mathbb{E}[f(\bm{X}')])^2}{\sigma_0^2} \right ) \right ] & (\mathbb{E}[f(\bm{X}')] = 0) \\
        &\leq \mathbb{E}\left [\exp \left( \frac{(f(\bm{X}) - f(\bm{X}'))^2}{\sigma_0^2} \right ) \right ] & (\text{Jensen}) \\
        &\leq \mathbb{E}\left [\exp \left( \frac{\norm{\bm{X} - \bm{X}'}^2}{\sigma_0^2} \right ) \right ] & (f \text{ 1-Lip.}) \\
        &\leq \mathbb{E}\left [\exp \left( \frac{2\norm{\bm{X}}^2 + 2\norm{\bm{X}'}^2}{\sigma_0^2} \right ) \right ] & (\text{Jensen}) \\
        &= \mathbb{E}\left [\exp \left( \frac{2\norm{\bm{X}}^2}{\sigma_0^2} \right ) \right ]^2 & (\text{Independence})
    \end{align*}
    Note that if $\sigma_0^2 = 4d\norm{\bm{X}}_{\psi_2}^2$ then
    \begin{equation*}
        \mathbb{E}\left [\exp \left( \frac{2\norm{\bm{X}}^2}{\sigma_0^2} \right ) \right ]^2 = \mathbb{E}\left [\exp \left( \frac{\norm{\bm{X}}^2}{2d\norm{\bm{X}}_{\psi_2}^2} \right ) \right ]^2 \leq 2^2 = 4
    \end{equation*}
    where the bound follows from the definition of $\norm{\bm{X}}_{\psi_2}$. By applying Lemma \ref{lem:change_constant} with $K=4$ this shows that
    \begin{equation*}
        \mathbb{E}[e^{f(\bm{X})^2/8d\norm{\bm{X}}_{\psi_2}^2}] \leq 2.
    \end{equation*}    
\end{proof}

Combining the previous two results we can prove the required bound on the Laplace functional as follows.
\begin{proof}[Proof of Lemma \ref{lem:subg_lap_bound}]
    Let $f$ be any mean-zero 1-Lipschitz function. By Lemma \ref{lem:f_sq_bound} we have
    \begin{equation*}
        \mathbb{E}[e^{f(\bm{X})^2/8d\norm{\bm{X}}_{\psi_2}^2}] \leq 2
    \end{equation*}
    and therefore by Lemma \ref{lem:sq_mz_to_not_sq} applied to the random variable $f(\bm{X})$ with $\sigma^2 = 8d\norm{\bm{X}}_{\psi_2}^2$ we have
    \begin{equation*}
        \mathbb{E}[\exp(\lambda f(\bm{X}))] \leq \exp (8d\norm{\bm{X}}_{\psi_2}^2\lambda^2).
    \end{equation*}
    Since this bound holds for all mean zero, 1-Lipschitz functions and all $\lambda \in \mathbb{R}$ we have
    \begin{equation*}
        E_\mu(\lambda) = \sup_f \mathbb{E}[\exp(\lambda f(\bm{X}))] \leq \sup_f \exp(8d\norm{\bm{X}}_{\psi_2}^2\lambda^2) = \exp(8d\norm{\bm{X}}_{\psi_2}^2\lambda^2)
    \end{equation*}
    which proves the bound.
\end{proof}

\subsection{Proof of Lemma \ref{lem:expected_squared_concentration}} \label{sec:expected_squared_concentration}

\begin{proof}
    Note that by the law of total expectation we have
    \begin{equation*}
        2 \geq \mathbb{E} \left [ \exp \left ( \frac{\norm{\bm{Y}}^2}{2d\norm{\bm{Y}}_{\psi_2}^2} \right ) \right ] = \mathbb{E}_{\bm{X}} \left [ \mathbb{E}  \left [ \exp \left ( \frac{\norm{\bm{Y}^{\bm{X}}}^2}{2d\norm{\bm{Y}}_{\psi_2}^2} \right ) \right ] \right ].
    \end{equation*}
    Since the expectation is finite we have with probability 1 (with respect to $\bm{X}$) that
    \begin{equation*}
         \mathbb{E}  \left [ \exp \left ( \frac{\norm{\bm{Y}^{\bm{X}}}^2}{2d\norm{\bm{Y}}_{\psi_2}^2} \right ) \right ]  < \infty.
    \end{equation*}
    In particular we can use Lemma \ref{lem:change_constant} to show that for $\mu$ almost every $\bm{x}$
    \begin{equation*}
        \norm{\bm{Y}^{\bm{x}}}_{\psi_2} \leq \max \left ( \norm{\bm{Y}}_{\psi_2}, \norm{\bm{Y}}_{\psi_2} \sqrt{\frac{\log \mathbb{E}\left [ \exp \left ( (\norm{\bm{Y}^{\bm{x}}}^2 / 2d\norm{\bm{Y}}^2_{\psi_2} \right ) \right]}{\log(2)}} \right ).
    \end{equation*}
    For convenience define the set 
    \begin{equation*}
        A = \left \{ \bm{x} \in \mathbb{R}^d \ \bigg | \  \mathbb{E}\left [ \exp \left ( \norm{\bm{Y}^{\bm{x}}}^2 / 2d\norm{\bm{Y}}^2_{\psi_2} \right ) \right] > 2 \right \}.
    \end{equation*}
    Now squaring and taking expectations on both sides we have
    \begin{align*}
        \mathbb{E}_{\bm{X}}\left [\norm{\bm{Y}^{\bm{X}}}_{\psi_2}^2 \right ] &\leq \norm{\bm{Y}}_{\psi_2}^2 \mathbb{E}_{\bm{X}} \left [ \max\left ( 1, \sqrt{\frac{\log \mathbb{E}\left [ \exp \left ( \norm{\bm{Y}^{\bm{X}}}^2 / 2d\norm{\bm{Y}}^2_{\psi_2} \right ) \right]}{\log(2)}} \right )^2 \right ] \\
        &= \norm{\bm{Y}}_{\psi_2}^2 \mathbb{P}(\bm{X} \notin A) + \norm{\bm{Y}}_{\psi_2}^2\mathbb{E} \left [ \pmb{1}[\bm{X} \in A] \frac{\log \mathbb{E}\left [ \exp \left ( \norm{\bm{Y}^{\bm{X}}}^2 / 2d\norm{\bm{Y}}^2_{\psi_2} \right ) \right]}{\log(2)} \right ] 
    \end{align*}
    where we have used that when $\bm{X} \in A$ the maximum occurs at the second term and when $\bm{X} \not \in A$ the maximum is achieved by the first. 

    To control the latter term we have
    \begin{align*}
        \mathbb{E}_{\bm{X}} \left [ \pmb{1}[\bm{X} \in A] \frac{\log \mathbb{E}\left [ \exp \left ( \norm{\bm{Y}^{\bm{X}}}^2 / 2d\norm{\bm{Y}}^2_{\psi_2} \right ) \right]}{\log(2)} \right ] 
        &\leq \mathbb{E}_{\bm{X}} \left [ \frac{\log \mathbb{E}\left [ \exp \left ( \norm{\bm{Y}^{\bm{X}}}^2 / 2d\norm{\bm{Y}}^2_{\psi_2} \right ) \right]}{\log(2)} \right ] \\
        &\leq \frac{1}{\log 2} \log \mathbb{E}_{\bm{X}} \left [  \mathbb{E}\left [ \exp \left ( \norm{\bm{Y}^{\bm{X}}}^2 / 2d\norm{\bm{Y}}^2_{\psi_2} \right ) \right] \right ] \\
        &= \frac{1}{\log 2}\log \mathbb{E} \left [ \exp (\norm{\bm{Y}}^2/2d\norm{\bm{Y}}^2_{\psi_2}) \right ] \\
        &\leq\frac{\log 2}{\log 2} = 1
    \end{align*}

    Overall this gives the bound
    \begin{equation*}
        \mathbb{E}_{\bm{X}}\left [ \norm{\bm{Y}^{\bm{X}}}_{\psi_2}^2 \right ] \leq  (1 + \mathbb{P}(\bm{X} \notin A)) \norm{\bm{Y}}_{\psi_2}^2 \leq 2 \norm{\bm{Y}}_{\psi_2}^2
    \end{equation*}
\end{proof}

\section{Details of Section \ref{sec:numerics} and Further Experiments} \label{sec:numeric_details}

The simulations heavily utilize the Python Optimal Transport library \cite{flamary2021pot} to solve the optimization problems involved in the simulations.

\subsection{Gaussian Experiments}

To make Figure \ref{fig:exact_gaussian_eps1} we set $\varepsilon=1.0$. We chose the paramters
\begin{equation*}
    \bm{x}_0 = \begin{bmatrix}
        0 \\ 0 \\ 0 \\ 0 \\ 0
    \end{bmatrix} \hspace{0.25cm}
    \Sigma_0 = \begin{bmatrix}
        1 & 0 & 0 & 0 & 0 \\
        0 & 1 & 0 & 0 & 0 \\
        0 & 0 & 1 & 0 & 0 \\
        0 & 0 & 0 & 1 & 0 \\
        0 & 0 & 0 & 0 & 1
    \end{bmatrix} \hspace{0.25cm}
    \bm{x}_1 = \begin{bmatrix}
        2 \\ 1 \\ 0 \\ -1 \\ -2
    \end{bmatrix} \hspace{0.25cm}
    \Sigma_1 = \begin{bmatrix}
          2 & 0 & 0 & 0 & 0  \\  
          0 & 0.5 & 0 & 0 & 0  \\  
          0 & 0 & 1 & 0 & 0  \\  
          0 & 0 & 0 & 0.1 & 0  \\  
          0 & 0 & 0 & 0 & 5
    \end{bmatrix}.
\end{equation*}
Included as Figure \ref{fig:exact_gaussian_triple} is the same figure with $\varepsilon=2.0,5.0,10.0$.

\begin{figure}[htp]

\centering
\includegraphics[width=.3\textwidth]{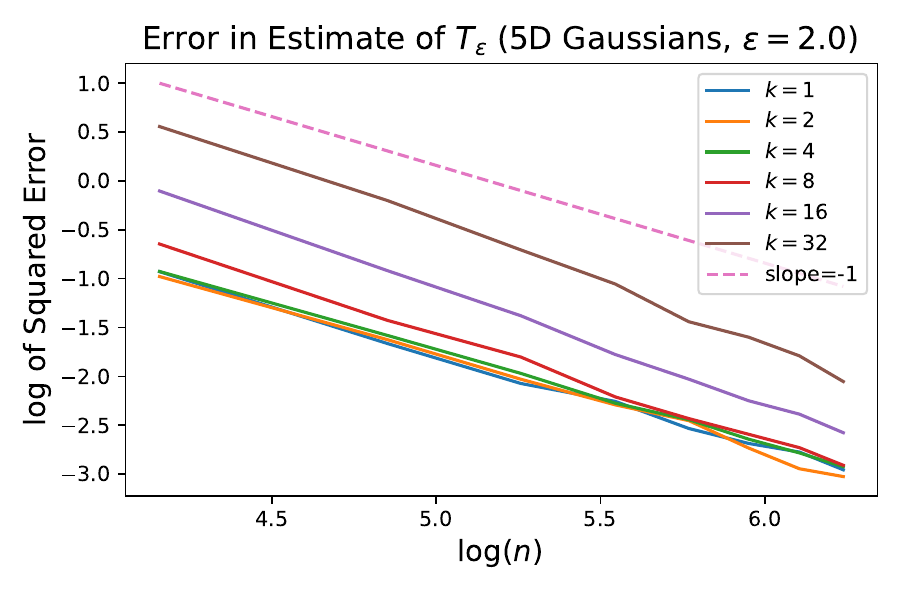}\hfill
\includegraphics[width=.3\textwidth]{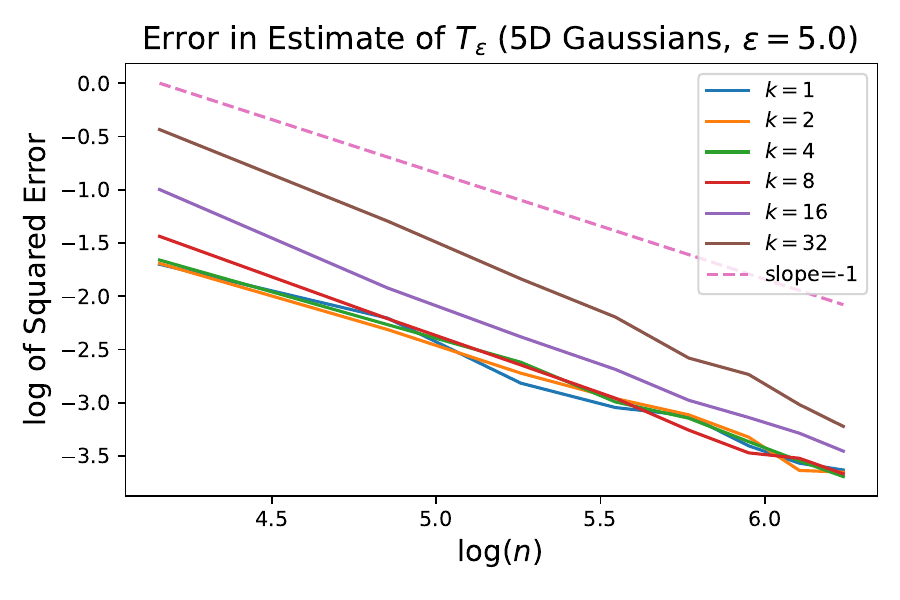}\hfill
\includegraphics[width=.3\textwidth]{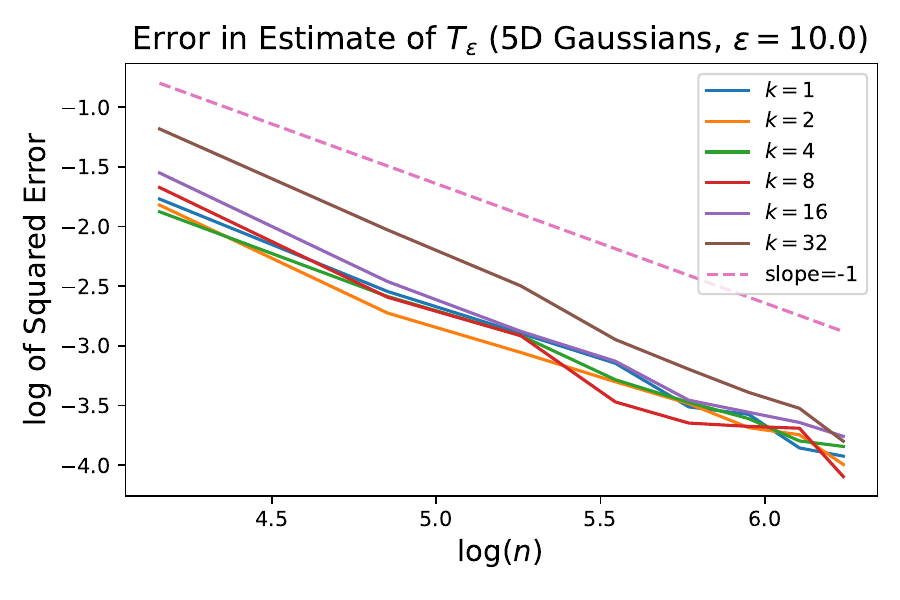}

\caption{Variations of Figure \ref{fig:exact_gaussian_eps1} with $\varepsilon=2,5$ and $10$ from left to right.}
\label{fig:exact_gaussian_triple}

\end{figure}

To highlight that this is not an artifact of the independence of the coordinates or the dimensionality we also include Figure \ref{fig:exact_change_d} which fixes $\varepsilon=2.0$ and uses Gaussians of varying dimension.
\begin{figure}
    \centering
    \includegraphics[width=0.5\textwidth]{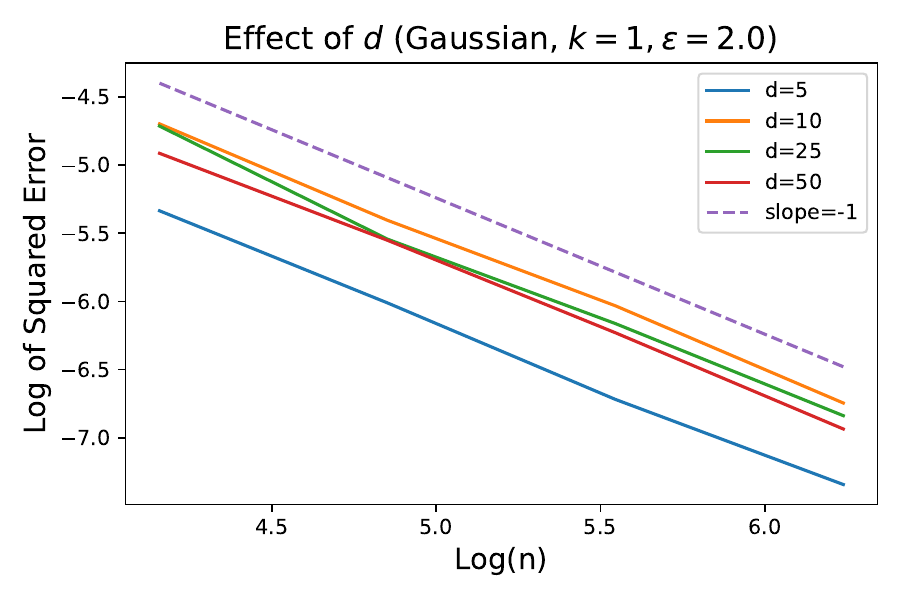}
    \caption{Effect of the dimension on the estimation of the exact map between two Gaussian measures.}
    \label{fig:exact_change_d}
\end{figure}
For this we create two random multivariate Gaussian measures $\mu = N(0,\Sigma), \nu = N(2e_1, \Sigma')$ in the following way. We generate $\Sigma$  by first selecting a random orthogonal matrix $U$, sampled uniformally from set of orthogonal matrices which serves as the eigenvectors of the covariance matrix. Them eigenvalues are independently sampled from $\text{Unif}([1/(10d),1/d])$ and embedded into a diagonal matrix $D$. From $U$ and $D$ we construct $\Sigma = UDU^T$. $\Sigma'$ is constructed in the same way with independently drawn $U',D'$.

The factor of $1/d$ is required to preserve the relative scale of the data as the dimension grows. A direct calculation shows that if $\bm{X},\bm{X}' \sim N(0,I)$ then
\begin{equation*}
    \mathbb{E}[\norm{\bm{X} - \bm{X}'}^2] = 2d
\end{equation*}
and if $\bm{Y},\bm{Y}' \sim N(0,I/d)$ then
\begin{equation*}
    \mathbb{E}[\norm{\bm{Y} - \bm{Y}'}^2] = 2.
\end{equation*}
This normalization allows the choice of $\varepsilon$ to be held fixed as the dimension varies and is a recurring theme in the experiments below which vary the dimension $d$.

This figure was made using 100 replications and 500 samples for the Monte Carlo integration.

\subsection{Strongly Log-Concave}

In order to create interesting strongly log-concave measures we consider densities of the form
$$\exp \left ( -c\frac{\norm{\bm{x}}^2}{2} + h(\bm{x})) \right )$$
where $c > 0$ and $h:\mathbb{R}^d \rightarrow \mathbb{R}$ is a convex function. These functions are guaranteed to be at least $c$-strongly log-concave. 

For the function $h$ we consider functions of the form
\begin{equation*}
    h(\bm{x}) = \max_{i=1,...,\ell} \bm{u}_i^T\bm{x} + b_i
\end{equation*}
where $\bm{u}_i \in \mathbb{R}^d$ and $b_i \in \mathbb{R}$. These are always convex functions regardless of the values of $\bm{u}_i$ and $b_i$. 

To generate Figure \ref{fig:variance_lcc} we use for the density of $\mu$ that $c = 1.0$ and generate $\bm{u}_i \sim N(0,I)$ and $b_i \sim N(0,1)$ for $i=1,...,20$. For $\nu$ we chose $c = 0.75$ and a further 20 terms $\bm{u}_i' \sim N(0,I)$ and $b_i' \sim N(0,1)$. For each line in the plot we used a different random draw of $\mu$ and $\nu$, and kept the parameters fixed along each line.

In order to generate samples from $\mu$ and $\nu$ in this case we used the Metropolis-adjusted Langevin Algorithm \cite{dwivedi2018log} and ran for 500 iterations for each sample with  a step size of 0.01.

We also explore the impact the dimension in this setting and recreate Figure \ref{fig:variance_lcc} with $\varepsilon=5$ fixed and $d=5,10,25$ in Figure \ref{fig:lcc_d}. We generate the slopes and intercepts in the same way as above except with $4d$ slopes and intercepts, and we scale the constant $c$ by a factor of $5/d$ to account for the scaling of the distance with the dimension. This figure was made using 100 replications and 500 samples for the Monte Carlo integration.
\begin{figure}
    \centering
    \includegraphics[width=0.5\textwidth]{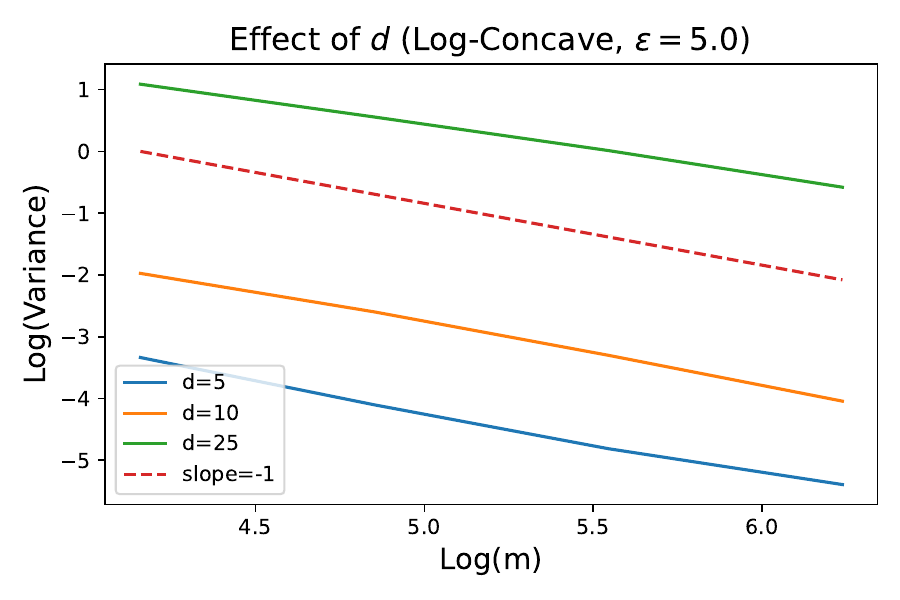}
    \caption{Impact of the dimension on the convergence of the variance term in the log-concave setting.}
    \label{fig:lcc_d}
\end{figure}

\subsection{Gaussian Mixture Models}

To create Figure \ref{fig:variance_gmm} we consider Gaussian mixture models of the form $\mu = \sum_{i=1}^\ell \alpha_i N(\bm{x}_i,\Sigma_i), \nu = \sum_{j=1}^{\ell'} \beta_j N(\bm{y}_j,\Sigma_j')$ where $\alpha_i, \beta_j > 0$ with $\sum_i \alpha_i = 1, \sum_{j} \beta_j = 1$, $\bm{x}_i,\bm{y}_j \in \mathbb{R}^d$, and $\Sigma_i,\Sigma_j' \in \mathbb{R}^{d \times d}$ are PSD matrices. The specific parameters used to create Figure \ref{fig:variance_gmm} are
\begin{gather*}
    \alpha = (0.25,0.25,0.25,0.25), \hspace{1cm} \beta = (0.1,0.2,0.3,0.4) \\
    \bm{x}_1 = \begin{bmatrix}
        0 \\ 0 \\ 0 \\ 0 \\ 0
    \end{bmatrix},
    \bm{x}_2 = \begin{bmatrix}
        1 \\ -1 \\ -1 \\ -1 \\ -1 
    \end{bmatrix},
    \bm{x}_3 = \begin{bmatrix}
        -1 \\ -1 \\ -1 \\ -1 \\ 1
    \end{bmatrix},
    \bm{x}_4 = \begin{bmatrix}
        0 \\ 0 \\ 3 \\ 0 \\ 0
    \end{bmatrix} \\
    \bm{y}_1 = \begin{bmatrix}
        1 \\ 0 \\ 0 \\ 0 \\ 0
    \end{bmatrix},
    \bm{y}_2 = \begin{bmatrix}
        0 \\ 1 \\ 0 \\ 0 \\ 0
    \end{bmatrix},
    \bm{y}_3 = \begin{bmatrix}
        0 \\ 0 \\ 1 \\ 0 \\ 0
    \end{bmatrix},
    \bm{y}_4 = \begin{bmatrix}
        0 \\ 0 \\ 0 \\ 1 \\ 0
    \end{bmatrix} \\
    \Sigma_1 = \frac{1}{2} I, \ \ \Sigma_2 = \frac{1}{5} I, \ \ \Sigma_3 = \begin{bmatrix}
        0.5 & 0 & 0 & 0 & 0 \\
        0 & 1.0 & 0 & 0 & 0 \\
        0 & 0 & 1.5 & 0 & 0 \\
        0 & 0 & 0 & 1.0 & 0 \\
        0 & 0 & 0 & 0 & 0.5 
    \end{bmatrix}, \ \ \Sigma_4 = \begin{bmatrix}
        0.1 & 0 & 0 & 0 & 0 \\
        0 & 0.1 & 0 & 0 & 0 \\
        0 & 0 & 3.0 & 0 & 0 \\
        0 & 0 & 0 & 0.1 & 0 \\
        0 & 0 & 0 & 0 & 0.1 
    \end{bmatrix} \\
    \Sigma_1' = I, \ \Sigma_2' = \frac{1}{2}I, \ \Sigma_3' = \frac{1}{4}I, \ \Sigma_4' = \frac{1}{8}I.
\end{gather*}

In a similar way to the log-concave setting we also explore the impact the dimension has in this setting. We generate a random source measure $\mu$ by using 10 points $\bm{x}_1,...,\bm{x}_{10} \sim N(0,(1/d)I)$ with $\alpha \sim \text{Dirichlet}(1,...,1)$. The covariance matrices are given by $\Sigma_i = (\bm{w}_i/d)I$ where $\bm{w}_i \sim \text{Unif}([0,1])$. For $\nu$ we use 15 points with the $\bm{y}_1,...,\bm{y}_{15} \sim N(2e_1,(1/d)I)$ where $e_1 = [1,0,...,0]$ and the weights and covariances generated in the same way as $\mu$. The results are displayed in Figure \ref{fig:gmm_change_d}. This figure was made using 100 replications and 500 samples for the Monte Carlo integration.

\begin{figure}
    \centering
    \includegraphics[width=0.5\textwidth]{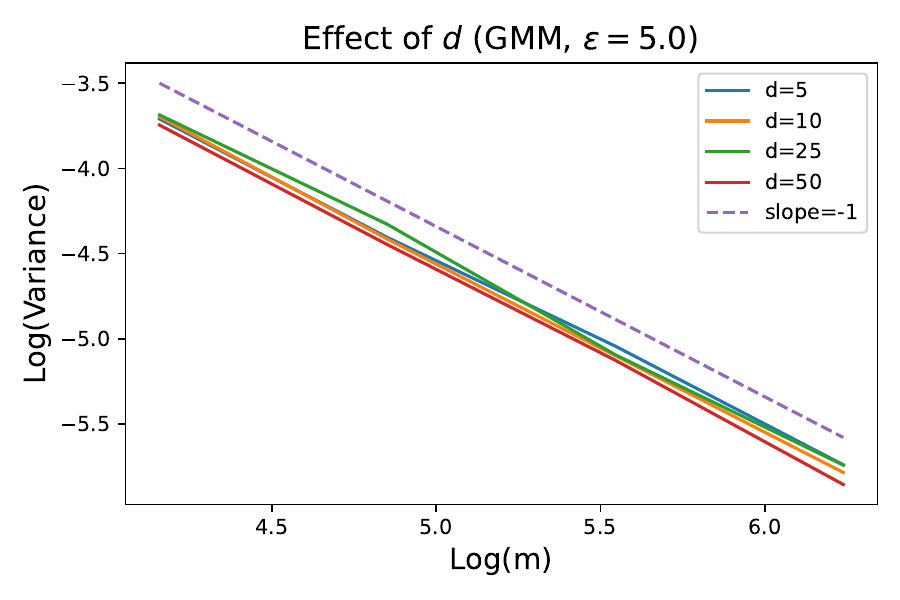}
    \caption{Impact of the dimension on the convergence of the variance term for Gaussian Mixture Models.}
    \label{fig:gmm_change_d}
\end{figure}

\end{document}